\documentclass[10pt]{article}

\usepackage{setspace}
\usepackage{amsfonts,bm,xspace,amsthm,fullpage,amssymb,dsfont}
\usepackage{multirow,rotating}
\usepackage{graphicx}

\usepackage{url,cite}
\usepackage{enumerate}
\usepackage{xcolor,soul}
\usepackage[footnotesize,bf,margin=1cm]{caption}

\usepackage[utf8]{inputenc} 
\usepackage[T1]{fontenc}
\usepackage{hyperref}
\usepackage{standalone}
\usepackage{verbatim}
\usepackage{setspace}
\usepackage{paralist}
\usepackage{mdwmath}
\usepackage{amssymb}
\usepackage{amsmath}
\usepackage{amsthm}
\usepackage{xfrac}
\usepackage{mathtools}  
\usepackage{chngcntr}
\usepackage{tabulary}
\usepackage{booktabs}
\usepackage{epsfig}
\usepackage{wrapfig}
\usepackage{lipsum}
\usepackage{subcaption}
\usepackage{multirow}
\usepackage{footnote}
\usepackage{paralist}
\usepackage[linesnumbered,ruled]{algorithm2e}
\usepackage{float}

\newtheorem{theorem}{Theorem}
\newtheorem{definition}{Definition}

\newtheorem{remark}{Remark}
\newtheorem{lemma}[theorem]{Lemma}

\newtheorem{cor}[theorem]{Corollary}

\DeclareMathOperator*{\argmin}{arg\,min}

\newcommand{\RR}{\mathbb{R}}

\newcommand{\EE}{\mathbb{E}}

\newcommand{\PP}{\mathbb{P}}
\newcommand{\vecx}{{\bf x}}

\newcommand{\Exp}{\mathbb{E}}
\newcommand{\matG}{\mathbf{G}}
\newcommand{\matA}{\mathbf{A}}
\newcommand{\matB}{\mathbf{B}}

\newcommand{\ones}{{\bf 1}}
\newcommand{\eg}{{\it e.g.,} }
\newcommand{\ie}{{\it i.e.,} }
\newcommand{\err}{\operatorname{err}}

\numberwithin{equation}{section}

\begin{document}

\title{Approximate Gradient Coding via Sparse Random Graphs}
\author{
Zachary Charles$^\mu$, Dimitris Papailiopoulos$^\epsilon$, and Jordan Ellenberg$^\mu$\\
$^\mu$Department of Mathematics\\
$^\epsilon$Department of Electrical and Computer Engineering\\
University of Wisconsin-Madison
} 

\maketitle

\abstract{

Distributed algorithms are often beset by the straggler effect, where the slowest compute nodes in the system dictate the overall running time. 
Coding--theoretic techniques have been recently proposed to mitigate stragglers via algorithmic redundancy.
Prior work in {\it coded computation} and {\it gradient coding} has mainly focused on exact recovery of the desired output. However, slightly inexact solutions can be acceptable in applications that are robust to noise, such as model training via gradient-based algorithms.
In this work, we present computationally simple gradient codes based on sparse graphs that guarantee fast and approximately accurate distributed computation.
We demonstrate that sacrificing a small amount of accuracy can significantly increase algorithmic robustness to stragglers.
}

\section{Introduction}

Deploying algorithms on distributed systems has become the de facto choice for scaling-out machine learning on large data sets. 
Distributed algorithms generally achieve substantial runtime speedups
compared to single node algorithms. 
Unfortunately, these speedup gains often fall short of the theoretically optimal speedup when scaling beyond a few tens of compute nodes \cite{dean2012large,paleo}.
This commonly observed behavior is referred to as the {\it speedup saturation phenomenon}.  
One of the causes of speedup saturation is the presence of {\it stragglers}. These are compute nodes whose runtime is substantially higher than the average runtime among all nodes. 

Several methods for mitigating the effect of stragglers have been recently proposed. 
These approaches include replicating jobs across redundant nodes\cite{shah2016redundant} and dropping stragglers in the case that the underlying computation is robust to errors \cite{ananthanarayanan2013effective, wang2014efficient}.
Moreover, redundant compute nodes can help improve the performance of distributed model training algorithms, as demonstrated in \cite{chen2016revisiting}.

Recently, tools from coding theory have gained traction in an effort to mitigate stragglers.
Lee et al. \cite{lee2016speeding} proposed the use of techniques from coding theory to compensate for stragglers and communication bottlenecks in machine learning settings, especially for the computation of linear functions.	
Li et al. proposed using coding theory to reduce inter-server communication in the shuffling phase of MapReduce \cite{li2015coded}.  Then, \cite{reisizadehmobarakeh2017coded} propose another coding-theoretic algorithm for speeding up distributed matrix multiplication in heterogeneous clusters. The use of codes for distributed matrix multiplication and linear operations on functions was also studied in \cite{dutta2016short}, which analyzes the trade-off between the flexibility and sparsity of the code.

In \cite{tandon2016gradient}, the authors propose {\it gradient coding}, a technique to exactly recover the sum of gradients from a subset of compute nodes. The authors show that $k$ gradients can be recovered from any $k-s$ compute nodes, as long as each node computes $s+1$ gradients. In other words, the algorithm is robust to $s$ stragglers. Gradient coding is particularly relevant to synchronous distributed learning algorithms that involve computing sums of gradients, such as mini-batch stochastic gradient descent and full-batch gradient descent. While some of the gradient code constructions in \cite{tandon2016gradient} are randomized, the authors in \cite{raviv2017gradient} use deterministic codes based on expander graphs to achieve similar results. 
	
Most of the above results focus on {\it exact} reconstruction of a sum of functions. In many practical distributed settings, we may only require {\it approximate} reconstruction of the sum. For example, parallel model training in machine learning settings has been shown to be robust to noise \cite{mania2015perturbed}. In some scenarios, noisy gradients may even improve the generalization performance of the trained model \cite{neelakantan2015adding}. By only approximately reconstructing the desired function, we hope to increase the speed and tolerance to stragglers of our distributed algorithm. In fact, expander graphs, particularly Ramanujan graphs, can be used for such approximate reconstruction \cite{raviv2017gradient}. Unfortunately, expander graphs, especially Ramanujan graphs, can be expensive to compute in practice, especially for large numbers of compute nodes. Moreover, the desired parameters of the construction may be constrained according to underlying combinatorial rules. 

	\paragraph{Our Contributions}
	
	In this work, we use sparse graphs to create gradient codes capable of efficiently and accurately computing approximate gradients in a distributed manner. More generally, these codes can be used to approximately compute any sum of functions in a distributed manner. We formally introduce the approximate recovery problem using a coding--theoretic interpretation, and present and analyze two decoding techniques for approximate reconstruction: an optimal decoding algorithm that has polynomial-time complexity, and a fast decoding method that has linear complexity in the sparsity of the input.

	We focus on two different codes that are efficiently computable and require only a logarithmic number of tasks per compute node. The first is the Fractional Repetition Code (FRC) proposed in \cite{tandon2016gradient}. We show that FRCs can achieve small or zero error with high probability, even if a {\it constant fraction} of compute nodes are stragglers. However, we show that FRCs are susceptible to {\it adversarial stragglers}, where an adversary can force a subset of the nodes to become stragglers. To get around this issue, we also present the Bernoulli Gradient Code (BGC) and the regularized Bernoulli Gradient Code (rBGC), whose constructions are based on sparse random graphs. We show that adversarial straggler selection in general codes is NP-hard, suggesting that these random codes may perform better then FRCs against polynomial-time adversaries. We give explicit bounds on the error of BGCs and rBGCs that show that their potential tolerance to adversaries comes at the expense of a worse average-case error than FRCs. We provide simulations that support our theoretical results. These simulations show that there is a trade-off between the decoding complexity of a gradient code and its average- and worst-case performance.

\section{Setup}

	\subsection{Notation}
		In the following, we will denote vectors and matrices in bold and scalars in standard script. For a vector ${\bf v} \in \RR^m$, we will let $\|{\bf v}\|_2$ refer to its $\ell_2$-norm.
		For any matrix ${\bf A}$, we will let $\matA_{i,j}$ denote its $(i,j)$ entry and let ${\bf a}_j$ denote its $j$th column. We will also let $\|\matA\|_2$ denote its spectral norm while $\sigma_{\min}(\matA)$ will denote its smallest singular value.
		We will let $\ones_m$ denote the $m\times 1$ all ones vector, while $\ones_{n \times m}$ will denote the all ones $n\times m$ matrix. We define ${\bf 0_m}$ and ${\bf 0}_{n\times m}$ analogously. 

	\subsection{Problem Statement}

		In this work, we consider a distributed master-worker setup of $n$ compute nodes, each of which is assigned a maximum of $s$ tasks. The compute nodes can compute locally assigned tasks and they can send messages to the master node.

		\begin{figure}[h]
		\centering
		\includegraphics[width=0.5\columnwidth]{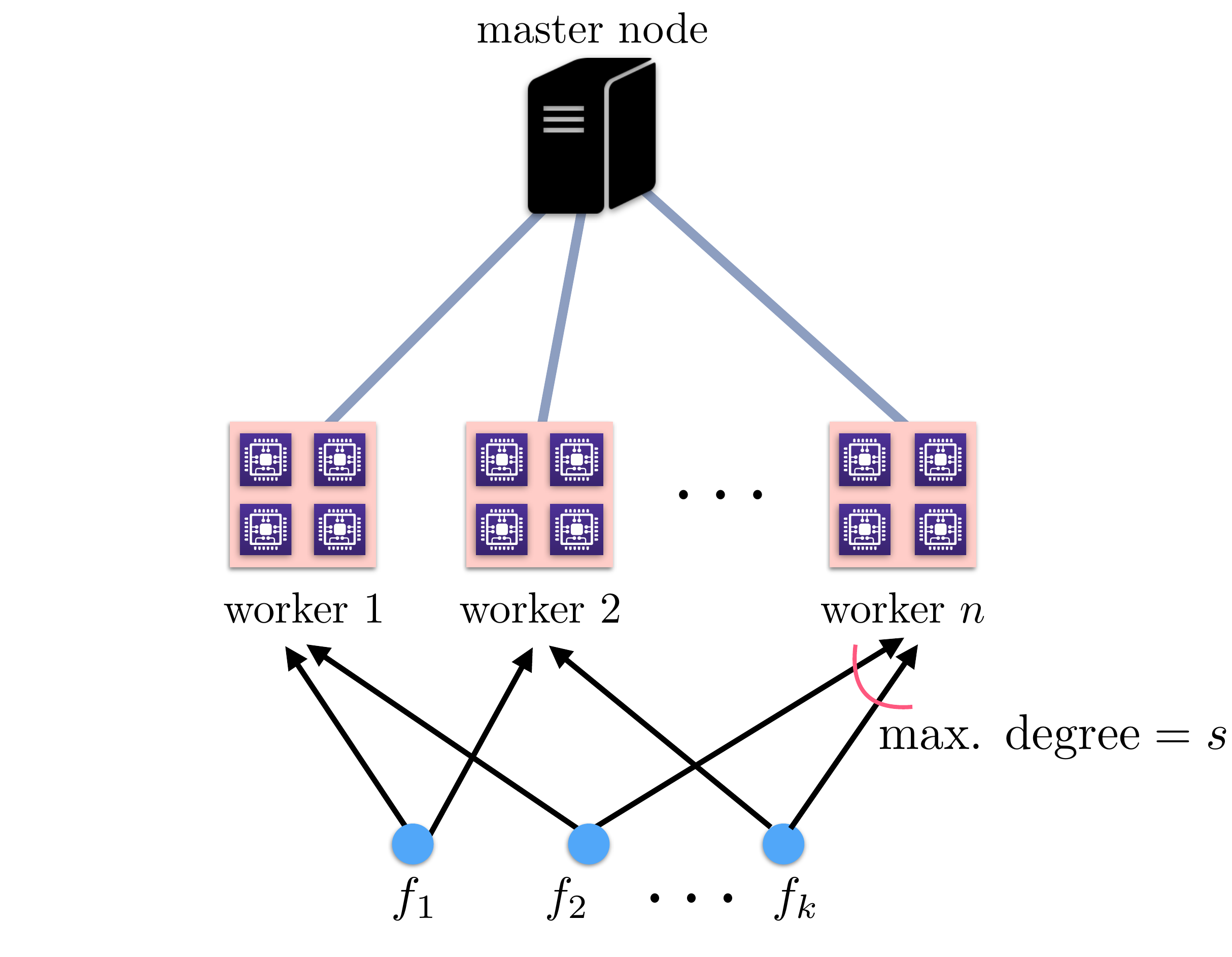}
		\caption{A master-worker architecture of distributed computation with multiple cores per machines. Each of the $k$ functions is assigned to a subset of the $n$ compute nodes.  
		Each compute node computes the set of assigned functions, and transmits a linear combination of them to the master node. 
		The master node aims to reconstruct the sum of the functions.}\label{fig:system}
		\end{figure}

		The goal of the master node is to compute the sum of $k$ functions
		\begin{equation}\label{eq:sum_functions}
		f({\bf x}) = \sum_{i=1}^k f_i({\bf x})	 = {\bf f}^T{\bf 1_k}
		\end{equation}
		in a distributed way, where $f_i:\mathbb{R}^d\rightarrow\mathbb{R}^w$, and any of the local functions $f_i$ can be assigned to and computed locally by any of the $n$ compute nodes. Here, ${\bf f} := [f_1(\vecx),\ldots, f_k(\vecx)]^T$. In order to keep local computation manageable, we assign at most $s$ tasks to each compute node. Due to the straggler effect, we assume that the master only has access to the output of $r < n$ non-straggler compute nodes. If we wish to exactly recover $f(\vecx)$, then we need $r \geq k-s+1$ \cite{tandon2016gradient}. However, in practice we only need to approximately recover $f(\vecx)$, which we may be able to do with many fewer non-stragglers.

		The above setup is relevant to distributed learning algorithms, where we often wish to find some model $\vecx$ by minimizing
		$$\ell(\vecx) = \sum_{i=1}^k \ell(\vecx;{\bf z}_i).$$
		Here $\{{\bf z}_i\}_{i=1}^k$ are our training samples and $\ell(\vecx;{\bf z})$ is a loss function measuring the accuracy of the model $\vecx$ with respect to the data point ${\bf z}$. 
		In order to find a model that minimizes the sum of losses $\ell(\vecx)$, we often use first-order, or gradient-based methods. For example, if we wanted to use gradient descent to find a good model, we would need to compute the full gradient of $\ell(\vecx)$, given by
		$$\nabla \ell(\vecx) = \sum_{i=1}^k \nabla \ell(\vecx;{\bf z}_i).$$
		Letting $f_i(\vecx) = \nabla \ell(\vecx;{\bf z}_i)$, we arrive at the setup in \ref{eq:sum_functions}. We focus our discussion on gradient descent for simplicity, but we note that the same setup applies for mini-batch SGD.

		Our main question is whether we can develop gradient coding techniques that are robust to larger numbers of straggler nodes if we allow for some error in the reconstruction of $\nabla \ell(x) $. We formally describe the approximate gradient coding problem, first considered in \cite{raviv2017gradient}.

		\medskip

		{\noindent\bf Approximate Gradient Coding:} 
		There are three components of an approximate gradient coding scheme: 
		\begin{enumerate}
			\item The function assignment per compute node.
			\item The messages sent from a compute node to the master.
			\item The decoding algorithm used by the master to recover an approximate sum of gradients.
		\end{enumerate}
		First, each compute node is assigned $s$ functions to compute locally. In some versions of the problem, we allow compute nodes to compute up to $\tilde{O}(s)$ functions, where  $\tilde{O}$ hides poly-log factors.

		After assigning tasks to each compute node, we let the compute nodes run local computations for some maximum amount of time. Afterwards, we may have compute nodes that either failed to compute some functions or are still running. These are the straggler nodes. 
		During the approximate reconstruction of the sum, we assume that the master node only has access to the output of the $r$ non-straggler nodes out of the $n$ total compute nodes. We want to use their output to compute the best approximation possible to $f(\vecx)$ given in (\ref{eq:sum_functions}). 
		We assume that we can only take linear combinations of the outputs of the non-straggler nodes.

		More formally, the task assignments are represented by a {\it function assignment matrix} ${\bf G}$, a $k\times n$ matrix where the support of column $j$ indexes the functions assigned to compute node $j$. The entries of column $j$ correspond to the coefficients of the linear combination of these local functions that the compute node sends back to the master once the compute node has completed its local computations.

		Let $\matA$ denote the $k\times r$ submatrix of ${\bf G}$ corresponding to the $r$ non-straggler compute nodes. The minimum recovery error for a given subset matrix $\matA$ is given by
		\begin{equation}\label{min1}\min_{\bf x}~({\bf f}^T{\bf A}{\bf x}-{\bf f}^T{\bf 1_k})^2.\end{equation}
		To better analyze this error, we define the {\it optimal decoding error} of a matrix $\matA$.
		\begin{definition}The optimal decoding error of a non-straggler matrix $\matA$ is defined as
		$$\err(\matA) := \min_{\vecx}\|\matA\vecx-\ones_k\|_2^2.$$\end{definition}

		The optimal decoding error quantifies how close $\ones_k$ is to being in the span of the columns of $\matA$. Taking $\vecx = {\bf 0}_r$, we see that for any $\matA$, $0 \leq \err(\matA) \leq k$. It is worth noting that $\err(\matA)$ is the {\it absolute error} incurred in our approximation. The multiplicative error is $\err(\matA)/k$. Note that if $\err(\matA)$ is small, then the overall minimum recovery error is small relative to $\|{\bf f}\|_2^2$, since
		\begin{equation}\label{min2}
			\min_{\bf x}~({\bf f}^T{\bf A}{\bf x}-{\bf f}^T{\bf 1_k})^2 \leq \|{\bf f}\|_2^2\err(\matA),
			\end{equation} 	
		For a given matrix $\matA$, let $\matA^+$ denote its pseudo-inverse. Properties of the pseudo-inverse imply
		$$\err(\matA) = \|\matA\matA^+\ones_k-\ones_k\|_2^2.$$

		In general, we are interested in constructing function assignment matrices ${\bf G}$ such that submatrices $\matA$ have small decoding error. Note that we can either consider the worst-case among all $\matA$ or consider the setting where $\matA$ is chosen uniformly at random. We will refer to these ${\bf G}$ matrices as {\it approximate gradient codes}.

		Approximate gradient codes were constructed in \cite{raviv2017gradient} using expander graphs. The authors show in particular that if ${\bf G}$ is the adjacency matrix of a Ramanujan graph, then the worst-case decoding error is relatively small. Unfortunately, such graphs may be expensive to compute in practice. To circumvent this issue, we use simplified random constructions.

		Our main theorem shows that there is an efficiently computable code that has small or zero decoding error with high probability, even with a constant fraction of stragglers. We state an informal version of this theorem below. Theorem \ref{simpler_interp} below will provide a more general and formal statement, along with the proof.

		\begin{theorem}\label{main_thm}
		We can assign $s = O(\log k)$ tasks to each compute node in such a way that with probability at least $1-\frac{1}{k}$, we can tolerate $\Theta(k)$ randomly chosen stragglers within a multiplicative error of $\epsilon$ where
		$$\epsilon = O\left(\dfrac{\log(k)}{k}\right).$$\end{theorem}

		In other words, for any $k$ and a constant fraction of stragglers $r$, there is a code with sparsity $\Theta(\log(k))$ that can exactly reconstruct the gradient with high probability. While this code achieves smaller error for most $\matA$ than previously designed gradient codes, we show that this comes at the expense of the worst-case error, which can be $\Theta(k)$. Here, the worst-case is taken over all possible sets of nodes that become stragglers. Moreover, this worst-case be computed efficiently by an adversary.
		On the other hand, we show that in general, adversarial straggler selection is NP-hard. In order to counter polynomial-time adversaries, we give another approximate gradient code that utilizes randomness. We also bound the decoding error of this code. We show the following informal theorem.

		\begin{theorem}\label{main_thm_bgc}
		We can randomly assign $O(\log k)$ tasks to each compute node in such a way that with probability at least $1-\frac{1}{k}$, we can tolerate $\delta k$ randomly chosen stragglers within a multiplicative error of $\epsilon$ where
		$$\epsilon = O\left(\dfrac{1}{(1-\delta)\log(k)}\right).$$\end{theorem}

		Theorem \ref{main_thm_4} below gives a more general and formal statement, along with the proof. In comparison, work in \cite{raviv2017gradient} derives worst-case bounds on $\err(\matA)$ when ${\bf G}$ is the adjacency matrix of an $s$-regular expander graph. Given such a ${\bf G}$, denotes its eigenvalues as $\lambda_1\geq \ldots \geq \lambda_k$. We will let $\lambda({\bf G}) := \max\{|\lambda_2|,|\lambda_k|\}$. Then the aforementioned work proves the following theorem.
		\begin{theorem}[\hspace{1sp}\cite{raviv2017gradient}]
		Suppose ${\bf G}$ is a $\log(k)$-regular expander. Then we can tolerate any $\delta k$ stragglers within a multiplicative error of $\epsilon$ where
		$$\epsilon = O\left(\dfrac{\delta\lambda({\bf G})^2}{(1-\delta)\log^2(k)}\right).$$\end{theorem}

		In particular, if ${\bf G}$ is a Ramanujan graph, then this becomes $\epsilon = O(\delta k/(1-\delta)s)$. This comes at the expense of having a more computationally difficult construction, as expander graphs, especially Ramanujan graphs, can be difficult to compute.

		\medskip

		{\noindent\bf Decoding:} We would like to note that the performance of a gradient code depends in part on the {\it decoding algorithm}, \ie how we use the output from the non-stragglers to approximate the desired output. In our setting, we want to use the received non-straggler matrix $\matA$ to approximate $\ones_k$. We give two possible methods below. We will refer to these as {\it decoding methods} because of the parallels to coding theory.
 
		\medskip

		\begin{minipage}{.45\linewidth}
			\begin{algorithm}[H]
			    \SetKwInOut{Input}{Input}
			    \SetKwInOut{Output}{Output}

			    \Input{A $k\times r$ non-straggler matrix $\matA$, $\rho > 0$.}
			    \Output{An approximation ${\bf v}$ to $\ones_k$}
			    $\vecx = \rho\ones_r$;\\
			    ${\bf v} = \matA\vecx$;\\
			    return ${\bf v}$;
			    \caption{One-step decoding.}
			\end{algorithm}
		\end{minipage}\hspace{1pt}	
		\begin{minipage}{.45\linewidth}
			\begin{algorithm}[H]
			    \SetKwInOut{Input}{Input}
			    \SetKwInOut{Output}{Output}

			    \Input{A $k\times r$ non-straggler matrix $\matA$.}
			    \Output{An approximation ${\bf v}$ to $\ones_k$}
			    $\vecx = \argmin \|\matA\vecx - \ones_k\|_2^2$;\\
			    ${\bf v} = \matA\vecx$;\\
			    return ${\bf v}$;
			    \caption{Optimal decoding.}
			\end{algorithm}
		\end{minipage}%

		\medskip

		For the one-step decoding error, we will generally consider $\rho = \frac{k}{rs}$. If ${\bf G}$ has $s$ entries in each column and row, then we would expect $\matA$ to have roughly $\frac{rs}{k}$ entries in each row. If this holds exactly, then setting $\rho = \frac{k}{rs}$ will allow us to exactly reconstruct the gradient.

		A decoding method analogous to the one-step decoding method was previously used in \cite{raviv2017gradient}. Note that the one-step decoding method is more efficient to compute than the optimal decoding, especially when $\matA$ is ill-conditioned or $k$ is large. Moreover, we can apply the one-step decoding method even if we do not have direct access to $\matA$ but can compute matrix-vector product $\matA\vecx$. The one-step decoding method allows us to avoid putting the entire matrix $\matA$ in to memory of the master compute node in settings where this is not possible.

		It is straightforward to see that if ${\bf v}$ is the optimal decoding vector of $\matA$, then $\|{\bf v}-\ones_k\|_2^2 = \err(\matA)$. On the other hand, if ${\bf v}$ is the one-step decoding vector of $\matA$ then $\|{\bf v}-\ones_k\|_2^2 \geq \err(\matA)$. We define the {\it one-step decoding error} of $\matA$ as follows.
		\begin{definition}For a given $\rho > 0$, the one-step error of $\matA$ is defined by
		$$\err_1(\matA) := \|\rho\matA\ones_r-\ones_k\|_2^2.$$\end{definition}

	\section{Fractional Repetition Codes}

		We would like to devise a code that achieves small error with high probability in the setting that our stragglers are chosen randomly. In fact, this can be achieved by the {\it fractional repetition code} (FRC) used in \cite{tandon2016gradient}. Note that \cite{tandon2016gradient} only considers this code for exact reconstruction of the gradient over all subsets of stragglers. This code can still be used when we only want approximately reconstruct the sum of $k$ gradients with high probability.

		This scheme works by replicating certain tasks between compute nodes. Suppose that we have $k$ tasks and $k$ compute nodes and we want each compute node to compute $s$ tasks. Without loss of generality, we suppose that $s$ divides $k$. The assignment matrix ${\bf G}_{\text{frac}}$ for this scheme is then defined by
		$${\bf G}_{\text{frac}} = \begin{pmatrix} \ones_{s\times s} & {\bf 0}_{s\times s} & {\bf 0}_{s\times s} & \ldots & {\bf 0}_{s\times s}\\ {\bf 0}_{s\times s} & \ones_{s\times s} & {\bf 0}_{s\times s} & \ldots\\
		{\bf 0}_{s\times s} & {\bf 0}_{s\times s} & \ones_{s\times s} & \ldots & {\bf 0}_{s\times s}\\
		\vdots & \vdots & \vdots & \ddots & \vdots\\
		{\bf 0}_{s\times s} & {\bf 0}_{s\times s} & {\bf 0}_{s\times s} & \ldots & \ones_{s\times s}\end{pmatrix}.$$

		We assume that the $k\times r$ matrix ${\bf A}_{\text{frac}}$ of non-stragglers has columns that are sampled uniformly without replacement from the $k$ columns of ${\bf G}_{\text{frac}}$. We first compute the expected one-step decoding error. Let ${\bf a}_i$ denote column $i$ of ${\bf A}_{\text{frac}}$.

		\begin{lemma}\label{frac_lemma_1}
		$$\EE[{\bf a}_i^T{\bf a}_j] = \begin{cases} s &\vspace{5pt},~i = j\\
		\dfrac{s^2}{k}-\dfrac{s}{k} &\vspace{5pt},~i \neq j\end{cases}.$$\end{lemma}

		\begin{proof}Fix ${\bf a}_i$. Since ${\bf a}_i$ has $s$ non-zero entries that are all $1$, ${\bf a}_i^T{\bf a}_i = s$. Next, suppose $j \neq i$. By the construction of ${\bf G}_{\text{frac}}$, there are only $s-1$ columns of ${\bf G}_{\text{frac}}$ that are not orthogonal to ${\bf a}_i$. Note that ${\bf a}_j$ is a duplicate of ${\bf a}_i$ with probability $\frac{s-1}{k}$. If this holds, then ${\bf a}_i^T{\bf a}_j = s$, and it is 0 if this does not hold. Therefore, for $i \neq j$,
		$$\EE[{\bf a}_i^T{\bf a}_j] = s\left(\dfrac{s-1}{k}\right) = \dfrac{s^2}{k}-\dfrac{s}{k}.$$\end{proof}

		\begin{theorem}Setting $\rho = \frac{k}{rs}$ in the one-step decoding method, we have
		$$\EE\left[\err_1(\matA_{\text{frac}})\right] = \dfrac{\delta k}{(1-\delta)s} - \dfrac{1}{1-\delta}\left(\dfrac{s-1}{s}\right).$$\end{theorem}

		\begin{proof}
		Using linearity of expectation, we have
		\begin{align*}
		\EE\left[\left\|\frac{k}{rs}\matA_{\text{frac}}\ones_r -\ones_k\right\|_2^2\right] &= \EE\left[\frac{k^2}{r^2s^2}\ones_r^T\matA_{\text{frac}}^T\matA_{\text{frac}}\ones_r - \frac{2k}{rs}\ones_k^T\matA_{\text{frac}}\ones_r + \ones_k^T\ones_k\right]\\
		&= \frac{k^2}{r^2s^2}\ones_r^T\EE[\matA_{\text{frac}}^T\matA_{\text{frac}}]\ones_r - \frac{2k}{rs}s\ones_r^T\ones_r + k\\
		&= \frac{k^2}{r^2s^2}\sum_{i, j}\EE[{\bf a}_i^T{\bf a}_j] -k.\end{align*}

		Between step 1 and step 2 we used the fact that the columns of $\matA_{\text{frac}}$ all have $s$ non-zero entries so $\ones_k^T\matA_{\text{frac}} = s\ones_r$. Applying Lemma \ref{frac_lemma_1},
		\begin{align*}
		\EE\left[\left\|\frac{k}{rs}\matA_{\text{frac}}\ones_r -\ones_k\right\|_2^2\right] &= \frac{k^2}{r^2s^2}\sum_{i, j}\EE[{\bf a}_i^T{\bf a}_j] -k\\
		&= \frac{k^2}{r^2s^2}\left(rs + r(r-1)\left(\frac{s^2}{k}-\frac{s}{k}\right)\right) - k\\
		&= \frac{k^2}{r^2s^2}\left(rs + \frac{r^2s^2}{k} - \frac{r^2s}{k}-\frac{rs^2}{k}+\frac{rs}{k}\right)-k\\
		&= \frac{k^2}{rs} - \frac{k}{s}-\frac{k}{r}+\frac{k}{rs}\\
		&= \frac{k}{s}\left(\frac{k}{r}-1\right) -\frac{k}{r}+\frac{k}{rs}\\
		&= \frac{\delta k}{(1-\delta)s} - \frac{1}{1-\delta}\left(\frac{s-1}{s}\right).
		\end{align*}\end{proof}

		Next, we consider the optimal decoding error of $\matA_{\text{frac}}$. Note that each column of $\matA_{\text{frac}}$ must be equal to one of the following $k/s$ distinct vectors,
        $${\bf v}_1 = \begin{pmatrix}\ones_{s}\\ {\bf 0}_s \\ \vdots \\{\bf 0}_s\end{pmatrix}, {\bf v}_2 = \begin{pmatrix} {\bf 0}_s \\ \ones_s \\ \vdots \\{\bf 0}_s\end{pmatrix},\ldots, {\bf v}_{k/s} = \begin{pmatrix}{\bf 0}_s \\ {\bf 0}_s \\ \vdots \\ \ones_s\end{pmatrix}.$$

        There are $s$ copies of each ${\bf v}_i$ in ${\bf G}_{\text{frac}}$ and $\matA_{\text{frac}}$ has a set of columns given by sampling $r$ of these without replacement. It is straightforward to see that $\err(\matA_{\text{frac}}) = \alpha s$, where $\alpha$ is the number of $i$ such that ${\bf v}_i$ is not a column of $\matA_{\text{frac}}$.

        Let $Y_1,\ldots, Y_{k/s}$ denote the random variables where $Y_i$ indicates whether ${\bf v}_i$ is not a column of $\matA$. Note that we then have
        \begin{equation}\label{frac_rep_err}
        \err(\matA) = \sum_{i=1}^{k/s} sY_i.\end{equation}

        Each $Y_i$ is 1 iff none of the $s$ columns in the $i$th block of ${\bf G}_{\text{frac}}$ are sampled as part of the $r$ non-stragglers. Therefore,
        \begin{equation}\label{frac_rep_prob}
        \PP(Y_i = 1) = \dfrac{\binom{k-s}{r-s}}{\binom{k}{r}}.\end{equation}

        Combining (\ref{frac_rep_err}) and (\ref{frac_rep_prob}), we get the following theorem.
        \begin{theorem}
        $$\EE[\err(\matA_{\text{frac}})] = k\dfrac{\binom{k-s}{r-s}}{\binom{k}{r}}.$$\end{theorem}

        We would now like high-probability bounds on $\err(\matA)$. By (\ref{frac_rep_err}), this reduces to bounding how many of the $Y_i$ are non-zero. This can be done via standard techniques concerning with-replacement sampling.

        \begin{theorem}\label{frac_err_prob}If ${\bf G}_{\text{frac}}$ is $k\times k$, for all $\alpha \in \mathbb{Z}_{\geq 0}$,
        $$\PP\big(\err(\matA_{\text{frac}}) \leq \alpha s\big) \geq 1-\binom{k/s}{\alpha+1}\dfrac{\binom{k-(\alpha+1)s}{r}}{\binom{k}{r}}.$$\end{theorem}		

        \begin{proof}Fix $T \subseteq \{1,2,\ldots, k/s\}$ such that $|T| = \alpha+1$. Then 
        $$\PP\big(\forall i \in T,~{\bf v}_i\text{ is not a column of }\matA_{\text{frac}}\big) = \dfrac{\binom{k-(\alpha+1)s}{r}}{\binom{k}{r}}.$$

        Taking a union bound, this implies
        $$\bigcup_{T : |T| = \alpha+1} \PP\big(\forall i \in T,~{\bf v}_i\text{ is not a column of }\matA_{\text{frac}}\big) \leq \binom{k/s}{\alpha+1}\dfrac{\binom{k-(\alpha+1)s}{r}}{\binom{k}{r}}.$$

        Note that the probability that we have no more than $\alpha$ of the ${\bf v}_i$ missing from the columns of $\matA$ is the probability that $\err(\matA_{\text{frac}}) \leq \alpha s$. Therefore,
        $$\PP\big(\err(\matA_{\text{frac}}) \leq \alpha s\big) \geq 1-\binom{k/s}{\alpha+1}\dfrac{\binom{k-(\alpha+1)s}{r}}{\binom{k}{r}}.$$\end{proof}

        While this exact expression is complicated, this result easily shows that if $s = \Omega(\log(k))$, then with probability at least $1-\frac{1}{k}$, $\err(\matA)$ is relatively small. Recall that the number of non-stragglers $r = (1-\delta)k$ for $\delta \in (0,1)$.

        \begin{theorem}\label{simpler_interp}Suppose
        $$s \geq \left(1+\dfrac{1}{1+\alpha}\right)\dfrac{\log(k)}{1-\delta}.$$
        Then
        $$\PP\big(\err(\matA_{\text{frac}}) > \alpha s\big) \leq \frac{1}{k}.$$\end{theorem}

        \begin{proof}By Theorem \ref{frac_err_prob}, we have
        $$\PP\big(\err(\matA_{\text{frac}}) > \alpha s\big) \leq \binom{k/s}{\alpha+1}\dfrac{\binom{k-(\alpha+1)s}{r}}{\binom{k}{r}}.$$
        Simple estimates show
        \begin{align*}
        \dfrac{\binom{k-(\alpha+1)s}{r}}{\binom{k}{r}} &= \dfrac{(k-(\alpha+1)s)(k-(\alpha+1)s-1)\ldots (k-(\alpha+1)s-r+1)}{k(k-1)\ldots (k-r+1)}\\
        &\leq \left(\dfrac{(k-(\alpha+1)s)}{k}\right)^r.\end{align*}
        Therefore,
        $$\PP\big(\err(\matA_{\text{frac}}) > \alpha s\big) \leq k^{\alpha+1}\left(\dfrac{(k-(\alpha+1)s)}{k}\right)^r.$$
        We wish to show that for $s \geq (1+\frac{1}{1+\alpha})\log(k)/(1-\delta)$, the right-hand side of this equation is at most $\frac{1}{k}$. Manipulating, this is equivalent to $s$ satisfying
        \begin{equation}\label{suff_cond_frc_1}
        \dfrac{(\alpha+1)s}{k} \geq \left(1-k^{-(\alpha+2)/r}\right).\end{equation}
        Since $s \geq (1+\frac{1}{1+\alpha})\log(k)/(1-\delta)$, we have
        \begin{equation}\label{s_bound}
        \dfrac{(\alpha+1)s}{k} \geq \dfrac{(\alpha+2)\log(k)}{r}.\end{equation}
        Letting $\beta = (\alpha+2)\log(k)/r$, (\ref{s_bound}) implies that (\ref{suff_cond_frc_1}) holds if
        \begin{align*}
        \beta \geq 1-e^{-\beta}.\end{align*}
        Since this occurs for all $\beta \geq 0$, the desired result is shown.\end{proof}

        \begin{cor}Suppose
        $$s \geq \dfrac{2\log(k)}{1-\delta}.$$
        Then
        $$\PP\big(\err(\matA_{\text{frac}}) > 0\big) \leq \frac{1}{k}.$$\end{cor}

        Theorem \ref{simpler_interp} implies that with probability at least $1-\frac{1}{k}$, an FRC will have multiplicative error $\epsilon$ where
        $$\epsilon = O\left(\frac{s}{k}\right).$$
        Therefore, this implies Theorem \ref{main_thm}. While FRCs have demonstrably small optimal decoding error when the stragglers are selected randomly, we will later show that it does not perform well when the stragglers are selected adversarially. In order to improve our tolerance to adversarial stragglers, we will develop a coding scheme based on random graphs.

	\section{Adversarial Stragglers}

		\subsection{Adversarial Stragglers and Fractional Repetition Codes}\label{adv_straggler_frac}

			While FRCs have small average-case error, their worst-case error is large. Even worse, it is computationally efficient to find these worst-case straggler sets. In fact, they can be found in linear time in the number of compute nodes. Recall that the assignment matrix ${\bf G}_{\text{frac}}$ for FRC is defined by
			\begin{equation}\label{g_frac}
			{\bf G}_{\text{frac}} = \begin{pmatrix} \ones_{s\times s} & {\bf 0}_{s\times s} & {\bf 0}_{s\times s} & \ldots & {\bf 0}_{s\times s}\\ {\bf 0}_{s\times s} & \ones_{s\times s} & {\bf 0}_{s\times s} & \ldots\\
			{\bf 0}_{s\times s} & {\bf 0}_{s\times s} & \ones_{s\times s} & \ldots & {\bf 0}_{s\times s}\\
			\vdots & \vdots & \vdots & \ddots & \vdots\\
			{\bf 0}_{s\times s} & {\bf 0}_{s\times s} & {\bf 0}_{s\times s} & \ldots & \ones_{s\times s}\end{pmatrix}.\end{equation}

			As previously noted, each column of $\matA_{\text{frac}}$ has $k/s$ distinct possibilities,
			$${\bf v}_1 = \begin{pmatrix}\ones_{s}\\ {\bf 0}_s \\ \vdots \\{\bf 0}_s\end{pmatrix}, {\bf v}_2 = \begin{pmatrix} {\bf 0}_s \\ \ones_s \\ \vdots \\{\bf 0}_s\end{pmatrix},\ldots, {\bf v}_{k/s} = \begin{pmatrix}{\bf 0}_s \\ {\bf 0}_s \\ \vdots \\ \ones_s\end{pmatrix}.$$

			There are $s$ copies of each ${\bf v}_i$ in ${\bf G}_{\text{frac}}$. As long as one instance of ${\bf v}_i$ is a non-straggler, then we will recover these $s$ entries in $\ones_k$. If all $s$ are stragglers, then our optimal decoding vector $\tilde{\vecx} = \argmin_{\vecx} \|{\bf A}\vecx-\ones_k\|_2^2$ will have zeros in the corresponding $s$ entries. This contributes $s$ to the optimal decoding error. One can therefore see that if we wish to pick $s$ stragglers adversarially, then we would pick all of the $s$ copies of some ${\bf v}_i$ above.

			Furthermore, the optimal decoding error increases by $s$ if and only if all of the columns in one of the $k/s$ blocks of ${\bf G}_{\text{frac}}$ are all stragglers. Therefore, if one were to select $k-r$ stragglers adversarially, they would pick all of the $s$ columns from one block, then all of the $s$ columns from another block, and continue until they had selected $k-r$ stragglers. This corresponds to picking $r$ non-stragglers corresponding to every column of $r/s$ of the blocks in ${\bf G}_frac$.	Here, we assume that $s$ divides $r$ for simplicity.

			If ${\bf G}_frac$ is given as in (\ref{g_frac}), then we can simply select the first $r$ columns of ${\bf G}_{\text{frac}}$ to be non-stragglers. If ${\bf G}_{\text{frac}}$ is permuted, then we can simply select all columns corresponding to $r/s$ blocks. There are therefore $(k-r)/s$ blocks missing from $\matA$. Each contributes $s$ to the optimal decoding error. This implies that we have an overall error of $k-r$. The argument above shows that this is the worst-case error possible.

			Note that the adversary can find this set in $O(k)$ operations if they have full-knowledge that an FRC scheme is being used with presentation as in (\ref{g_frac}). Even if they do not have this knowledge, an adversary can check for this coding scheme and find the worst-case straggler set in $O(k^2)$ operations if they only have access to the matrix ${\bf G}_{\text{frac}}$. This implies the following theorem.			

			\begin{theorem}Suppose that we assign tasks using a Fractional Repetition Code. In the worst-case, the non-straggler matrix $\matA$ will satisfy
			$$\err(\matA) = k-r.$$
			Moreover, worst-case straggler sets can be found in quadratic time.\end{theorem}

			Suppose that $r = (1-\delta)k$ for some constant $\delta$. Then, the adversarial optimal decoding error is $\Theta(k)$. This is in stark contrast to Theorem \ref{simpler_interp}, which shows that if the stragglers are selected randomly and $s = \Omega(\log k)$, then with high probability $\err(\matA) = O(\log k)$. Also note that since $\err_1(\matA) \geq \err(\matA)$, the adversarial one-step decoding error is $\Omega(k-r)$.  

		\subsection{Adversarial Straggler Selection is NP-hard}

			In this section, we show that in general, adversarial straggler selection is NP-hard. This demonstrates that adversaries with polynomial-time computations may not be able to find a set of stragglers that maximizes the decoding error. In such cases, the average-case error may be a more useful indicator of how well a gradient code performs.

			To show that general adversarial selection is NP-hard, we first define two problems.
			\begin{definition}[Densest $k$-Subgraph Problem]The $k$-densest subgraph problem (D$k$S) asks, given a graph $(V, E)$, what $k$-vertex subgraph contains the most edges.\end{definition}

			As shown in \cite{asahiro2002complexity}, this problem is NP-hard, even if we restrict to regular graphs. We now formally define the adversarial straggler problem.

			\begin{definition}[Adversarial Straggler Problem]Fix a constant $\rho > 0$. The $r$-adversarial straggler problem ($r$-ASP) asks, given a square matrix ${\bf G} \in \real^{n\times n}$, which column-submatrix ${\bf A}$ maximizes
			$$\|\rho{\bf A}\ones_r-\ones_n\|_2^2.$$\end{definition}

			Note that this is the form that one-step decoding takes. We will show that for any $\rho \in (0,\frac{2}{3})$, this problem is NP-hard, even when we restrict to ${\bf G} \in \{0,1\}^{k\times k}$ with at most $s$ non-zero entries in each column.

			\begin{theorem}For any $\rho \in (0,\frac{2}{3})$, the adversarial straggler problem is NP-hard. This holds even if we restrict to matrices ${\bf G} \in \real^{k\times k}$ with entries in $\{0,1\}$ and at most $s \geq 2$ non-zero entries per column.\end{theorem}

			In fact, the proof of our theorem also shows that if we instead consider all matrices ${\bf G} \in \real^{k\times n}$ where $k \geq n$, then $r$-ASP is NP-hard for any $\rho > 0$.

			\begin{proof}We will give a reduction from D$k$S on $d$-regular graphs to $r$-ASP where ${\bf G}$ is boolean with at most $d$ non-zero entries per column.

			Let $(V, E)$ be a $d$-regular graph on $n$ vertices. Note that $|E| = nd$. We want to solve D$k$S for $(V, E)$. Let ${\bf M}$ denote the adjacency matrix of $(V, E)$. Note that D$k$S is equivalent to
			$$\max_{\vecx} \vecx^T{\bf M}\vecx\hspace{5pt}s.t.~\vecx\in\{0,1\}^{n},~\|\vecx\|_0=k.$$

			Let ${\bf B}$ denote the unsigned incidence matrix of $(V,E)$. That is, ${\bf B}$ is a $|E|\times |V|$ boolean matrix where the row corresponding to edge $e$ has a $1$ in column $v$ iff $e$ is incident to $v$. We will let ${\bf C}$ denote the $|E| \times |E|$ matrix given by adding $|E|-|V| = n(d-1)$ zero columns to ${\bf B}$. Note that ${\bf C}$ is a square $nd\times nd$ boolean matrix with at most $d$ non-zero entries in each column since $(V,E)$ is $d$-regular.

			Let $r = k+(n-1)d$. Consider $r$-ASP on ${\bf C}$. This is equivalent to finding a vector ${\bf x} \in \{0,1\}^{|E|}$ with $\|\vecx\|_0 = r$ that maximizes
			$$\|\rho{\bf C} \vecx - \ones_{nd}\|_2^2.$$

			Straightforward computations show
			\begin{align*}
			\|\rho{\bf C} \vecx - \ones_{nd}\|_2^2 &= \rho^2\vecx^T{\bf C}^T{\bf C}\vecx - 2\rho\ones_{nd}^T{\bf C}\vecx + \ones_{nd}^T\ones_{nd}\\
			&= \rho^2\vecx^T{\bf C}^T{\bf C}\vecx - 2\rho[d\ones_{n}^T~~{\bf 0}_{n(d-1)}^T] \vecx + nd.\end{align*}

			Let $\vecx = \begin{pmatrix}{\bf y}\\{\bf z}\end{pmatrix}$ where ${\bf y} \in \{0,1\}^{n}, {\bf z} \in \{0,1\}^{n(d-1)}$. Note that ${\bf y}$ corresponds to which of the columns of ${\bf B}$ we select, while ${\bf z}$ corresponds to columns of ${\bf 0}$ we select. Recall that $\|{\bf y}\|_0 + \|{\bf z}\|_0 = r = t+n(d-1)$.

			Note that ${\bf B}^T{\bf B} = {\bf M} + {\bf I}d$. This implies
			$${\bf C}^T{\bf C} = \begin{pmatrix} {\bf M} + {\bf I}d & {\bf 0}\\ {\bf 0} & {\bf 0}\end{pmatrix}.$$

			We then get
			\begin{align*}
			\|\rho{\bf C}\vecx - \ones_{nd}\|_2^2 &= \rho^2\vecx^T{\bf C}^T{\bf C}\vecx - 2\rho[d\ones_{n}^T~~{\bf 0}_{n(d-1)}^T] \vecx + nd\\
			&= \rho^2{\bf y}^T{\bf M}{\bf y} + d\rho^2 {\bf y}^T{\bf y} -2\rho d \ones_{n}^T{\bf y} + nd\\
			&= \rho^2{\bf y}^T{\bf M}{\bf y} + d\rho^2 \|{\bf y}\|_0 - 2\rho d\|{\bf y}\|_0 + nd.\end{align*}

			Therefore, $r$-ASP in this setting is equivalent to maximizing, over ${\bf y} \in \{0,1\}^n, {\bf z}\in \{0,1\}^{n(d-1)}$ such that $\|{\bf y}\|_0 + \|{\bf z}\|_0 = r = k+n(d-1)$, the quantity
			\begin{equation}\label{asp_obj}
			\rho^2{\bf y}^T{\bf M}{\bf y} + d\rho^2 \|{\bf y}\|_0 - 2\rho d\|{\bf y}\|_0 + nd.\end{equation}

			Define
			$$f({\bf y}) = \rho^2{\bf y}^T{\bf M}{\bf y} + (\rho^2-2\rho)d\|{\bf y}\|_0.$$

			Note that if we fix a binary ${\bf y}$ such that $\|{\bf y}\|_0 = a$, then
			$$f({\bf y}) = \rho^2{\bf y}^T{\bf M}{\bf y} + (\rho^2-2\rho)da.$$
			Therefore, maximizing this quantity corresponds to finding the $a$-densest subgraph of $(V,E)$. We now must show that when we maximize this over ${\bf y}$ and ${\bf z}$, then the solution will always have $\|{\bf y}\|_0 = k$. Since $r = k+n(d-1)$, it suffices to show that ${\bf y}$ is as sparse as possible.

			To show that this is the case, we will show that for $\rho \in (0,\frac{2}{3})$, increasing the sparsity of ${\bf y}$ by 1 will only decrease the objective function. Suppose that $\|{\bf y}\|_0 = a$ and it has support $S$. Say that ${\bf y}'$ satisfies $\|{\bf y}'\|_0 = a+1$ and it has support $S'$ where $S\subseteq S'$. Let $T, T'$ denote the vertex subgraphs of $(V,E)$ corresponding to $S, S'$, and let $e(S), e(S')$ denote the number of edges in these subgraphs. Note that ${\bf y}^T{\bf M}{\bf y} = 2e(S)$. We then have
			$$f({\bf y}) = 2\rho^2e(S) +(\rho^2-2\rho)da.$$
			$$f({\bf y}') = 2\rho^2e(S') + (\rho^2-2\rho)d(a+1).$$

			Note that since $(V,E)$ is $d$-regular, $e(S') \leq e(S) + d$. Therefore,
			\begin{align*}
			f({\bf y}') - f({\bf y}) &= 2\rho^2(e(S')-e(S)) + (\rho^2-2\rho)d\\
			&\leq 2\rho^2d + \rho^2d - 2\rho d\\
			&= 3\rho^2d - 2\rho d.\end{align*}

			For $\rho \in (0,\frac{2}{3})$, this quantity is negative. Therefore, increasing the sparsity of ${\bf y}$ will decrease the objective function $f({\bf y})$. Therefore, the maximum of the $r$-ASP problem applied to ${\bf C}$ will have ${\bf y}$ as sparse as possible. Since $r = k+n(d-1)$ and $\|{\bf z}\|_0 \leq n(d-1)$, this implies that the maximum occurs at $\|{\bf y}\|_0 = k$. Let $S$ denote the support of ${\bf y}$. The objective function in (\ref{asp_obj}) is then equal to
			\begin{equation}\label{asp_obj_2}
			2\rho^2e(S) + d\rho^2t-2\rho dt+nd.\end{equation}

			This is clearly maximized when $S$ is the set of vertices forming the densest $k$-subgraph.\end{proof}	     

	\section{Bernoulli Gradient Codes}

		In this section we will consider the case that ${\bf G}$ has entries that are Bernoulli random variables. For a given $s, k$, we will refer to the Bernoulli coding scheme as setting, for $i \in \{1,\ldots, k\}, j \in \{1,\ldots, n\}$, ${\bf G}_{i,j} = \text{Bernoulli}(s/k)$. Intuitively, by injecting randomness in to the construction of ${\bf G}$, we improve our tolerance to adversarial stragglers. This comes as the cost of worse average-case error. While \cite{raviv2017gradient} shows that if ${\bf G}$ is a Ramanujan graph then we have strong bounds on its adversarial decoding error, such graphs are notoriously tricky to compute. By using Bernoulli coding, we sacrifice a small amount of error in order to achieve a much simpler, efficiently computable coding scheme.

		Suppose that the stragglers are selected uniformly at random. Then, the non-straggler submatrix $\matA$ also has Bernoulli random entries. Note that the expected number of tasks assigned to each compute node is $s$. This construction will allow us to derive high-probability bounds on the decoding error for $s > \log(k)$. We will later show that we can enforce the desired sparsity $s$ of each column and maintain the same error. Moreover, enforcing this desired sparsity will let us extend these error bounds to the setting where $s < \log(k)$. In order to get a handle on the decoding error, we first develop a method to bound the optimal and one-step decoding errors.

		\subsection{Bounding the Decoding Error}
			Suppose we have a function assignment matrix ${\bf G}$ such that the sparsity of each column is exactly or approximately bounded by $s$. After performing the local computations on each compute node, we have access to a $k\times r$ submatrix $\matA$ of the $r$ non-stragglers. We assume that $r = (1-\delta)k$ for some $\delta \geq 0$.

			We would like to derive high probability bounds on $\err(\matA)$ in order to bound the optimal decoding error of $\matA$, as in (\ref{min2}). Unfortunately, it is not straightforward to directly bound this error for a random matrix $\matA$ since it involves the pseudo-inverse of $\matA$.
			Instead, we will use an algorithmic approach to bound the optimal decoding error with high probability. The following lemma is adapted from \cite{zouzias2013randomized}.

			\begin{lemma}\label{ut_lemma}
			Let ${\bf u}_0= {\bf 1}_k$, and define
			$$ {\bf u}_t = {\bf u}_{t-1} - \frac{{\bf A}{\bf A}^T }{\nu} {\bf u}_{t-1}.$$
			If $\nu \geq \|\matA\|_2^2$ then
			$$\lim_{t\rightarrow \infty} \|{\bf u}_t\|_2^2 = \err(\matA).$$
			Moreover, for all $t$, $\|{\bf u}_t\|_2^2 \geq \err(\matA)$.
			\end{lemma}

			We refer to the ${\bf u}_t$ as the {\it algorithmic decoding error} of $\matA$. To prove Lemma \ref{ut_lemma}, we will use the following lemma, adapted from \cite{zouzias2013randomized}.

			\begin{lemma}\label{proj_fact}
			If ${\bf u}$ is in the column span of $\matA$ and $\nu \geq \|\matA\|_2^2$ then
			$$\left\|\left({\bf I}-\frac{\matA\matA^T}{\nu}\right){\bf u}\right\|_2 \leq \left(1-\frac{\sigma_{\min}(\matA)^2}{\nu}\right)\|{\bf u}\|_2.$$\end{lemma}

			\begin{proof}[Proof of Lemma \ref{ut_lemma}]
			Fix some $t \geq 1$. We can decompose ${\ones_k}$ as ${\bf v} + {\bf w}$ where ${\bf v}$ is the orthogonal projection of $\ones_k$ on to the column span of $\matA$ and ${\bf w}$ is in the nullspace of $\matA^T$. Note that this implies that $({\bf I}-\matA\matA^T/\nu){\bf w} = {\bf w}$. Therefore,
			$${\bf u}_t = \left({\bf I}-\frac{\matA\matA^T}{\nu}\right)^t({\bf v}+{\bf w})= \left({\bf I}-\frac{\matA\matA^T}{\nu}\right)^t{\bf v} + {\bf w}.$$
			Since ${\bf v}$ is in the span of $\matA$, $\left({\bf I}-\frac{\matA\matA^T}{\nu}\right)^t{\bf v}$ is also in the span of $\matA$ and orthogonal to ${\bf w}$. By Lemma \ref{proj_fact},
			\begin{align*}
			\|{\bf u}_t\|_2^2 &= \left\|\left({\bf I}-\frac{\matA\matA^T}{\nu}\right)^t{\bf v}\right\|_2 + \|{\bf w}\|_2^2.\\
			&\leq \left(1-\frac{\sigma_{\min}(\matA)^2}{\nu}\right)^{2t}\|{\bf v}\|_2^2 + \|{\bf w}\|_2^2.\end{align*}
			By construction, $\|{\bf w}\|_2^2 = \min_{\vecx}\|\matA\vecx-\ones_k\|_2^2$, completing the proof.\end{proof}

			Note that the ${\bf u}_t$ are defined as the iterates of projected gradient descent. Consider the setting where $\nu = \|\matA\|_2^2$. The matrix ${\bf P} = \matA\matA^T/\|\matA\|_2^2$ is a projection operator (ie. ${\bf P}^2 = {\bf P}$), and it projects a vector on to the column-span of $\matA$. By letting ${\bf u}_t = {\bf u}_{t-1}-{\bf P}{\bf u}_{t-1}$, Lemma \ref{ut_lemma} one can show that this eventually converges to ${\bf u}_0 - \matA\matA^+{\bf u}_0$. In other words, we eventually converge to the component of ${\bf u}_0$ that is orthogonal to the range of $\matA$. Taking ${\bf u}_0 = \ones_k$, this eventually converges to $\err(\matA)$.

			We can better understand ${\bf u}_t$ by taking a combinatorial view. Note that $\matA$ encodes a bipartite graph with $k$ left vertices and $r$ right vertices, where $\matA_{ij}$ is 1 iff there is an edge between vertex $i$ on the left and vertex $j$ on the right. Column $j$ of $\matA$ corresponds to the incidence of the $j$th right vertex. In particular, the degree of vertex $j$ on the right equal the number of tasks computed by compute node $j$. We can compute $\|{\bf u}_t\|_2^2$ in terms of walks on this bipartite graph.

			\begin{lemma}\label{walk_lemma}$\ones_k^T(\matA\matA^T)^t\ones_k$ equals the number of paths of length $2t$ from a left vertex to a right vertex.\end{lemma}

			\begin{proof}Note that $(\matA\matA^T)_{ij}$ is the number of paths of length 2 from the vertex $i$ to vertex $j$, where $i, j$ are both left vertices. More generally, $(\matA\matA^T)^t_{ij}$ counts the weighted number of paths of length $2t$ from vertex $i$ to vertex $j$. Therefore, $\ones_k(\matA\matA^T)^t\ones_k$ is the weighted number of paths of length $2t$ from a left vertex to a left vertex.\end{proof}
				
			\begin{lemma}
			Let $a_t$ denote the weighted number of walks in the associated bipartite graph of $
			\matA$ of length $2t$ starting and ending at a left vertex. Then
			$$\|{\bf u}_t\|_2^2 = a_0 - \binom{2t}{1}\frac{a_1}{\nu} + \binom{2t}{2}\frac{a_2}{\nu^2} - \ldots + \binom{2t}{2t}\frac{a_{2t}}{\nu^{2t}}.$$
			\end{lemma}
			\begin{proof}By direct computation we have
			\begin{align*}
			\|{\bf u}_t\|_2^2 &= \left\|\bigg({\bf I}-\dfrac{\matA\matA^T}{\nu}\bigg){\bf u}_{t-1}\right\|_2^2\\
			&= \left\|\bigg({\bf I}-\dfrac{\matA\matA^T}{\nu}\bigg)^t\ones_k\right\|_2^2\\
			&= \ones^T_k\ones_k- \binom{2t}{1}\dfrac{\ones_k^T\matA\matA^T\ones_k}{\nu} + \binom{2t}{2}\dfrac{\ones_k^T(\matA\matA^T)^2\ones_k}{\nu^2} -\ldots + \binom{2t}{2t}\dfrac{\ones_k^T(\matA\matA^T)^{2t}\ones_k}{\nu^{2t}}.\end{align*}

			By Lemma \ref{walk_lemma}, the result follows.\end{proof}

			While ${\bf u}_t$ may be difficult to bound for sufficiently large $t$, we can handle ${\bf u}_1$ more directly. Moreover, as theory and simulations will show, even ${\bf u}_1$ will give us good bounds on $\err(\matA)$.		

		\subsection{One-step Error of Bernoulli Gradient Codes}

			Recall that our function assignment matrix ${\bf G}$ has entries that are Bernoulli random variables with probability $s/k$. The non-straggler matrix $\matA$ is a column submatrix and therefore also has Bernoulli random entries. We will view $\matA$ as encoding a bipartite graph with $k$ left vertices and $r$ right vertices. We say that there is an edge between left vertex $i$ and right vertex $j$ iff $\matA_{i,j} = 1$. Note that $\EE[\matA] = \frac{s}{k}\ones_{k\times r}$. Therefore, the expected degree of any vertex in the associated bipartite graph is at most $s$. We will bound $\|\matA-\EE \matA\|_2$ for various $s$ and use this to bound $\err(\matA)$. We will use the following lemma.

			\begin{lemma}\label{ebound}Suppose $\|\matA-\EE \matA\|_2 \leq \gamma$. If $\rho = \frac{k}{rs}$ then
			$$\err_1(\matA) \leq \dfrac{\gamma^2k}{(1-\delta)s^2}.$$\end{lemma}

			\begin{proof}
			By standard norm properties,

			\begin{align*}
			\left\|\frac{k}{rs}\matA\ones_r - \ones_k\right\|_2^2 &= \left\|\frac{k}{rs}\matA\ones_r - \frac{k}{rs}\Exp \matA \ones_r\right\|_2^2\\
			&\leq \frac{k^2}{r^2s^2}\|\matA - \Exp \matA\|_2^2\|\ones_r\|_2^2\\
			&\leq \frac{\gamma^2k^2}{rs^2}\\
			&= \frac{\gamma^2k}{(1-\delta)s^2}.\end{align*}\end{proof}

			This approach is analogous to bounding ${\bf u}_1$, as the following lemma shows.

			\begin{lemma}\label{u_err_lem}
			Suppose that $\|\matA-\EE\matA\|_2 \leq \gamma$ where $\gamma \leq \sqrt{(1-\delta)}s$. Then for $\nu = \frac{rs^2}{k}$,
			$${\bf u}_1 \leq \dfrac{5\gamma^2k}{(1-\delta)s^2}.$$\end{lemma}

			\begin{proof}Recall that for a given $\nu$, ${\bf u}_1$ is given by
			$${\bf u}_1 = \left\|\left({\bf I}-\dfrac{\matA\matA^T}{\nu}\right)\ones_k\right\|_2^2.$$

			We then have
			\begin{align*}
			{\bf u}_1 &= \left\|\ones_k-\dfrac{\matA\matA^T}{\nu}\ones_k\right\|_2^2\\
			&= \left\|\ones_k-\dfrac{\matA(\EE\matA^T+\matA^T-\EE\matA^T)}{\nu}\ones_k\right\|_2^2\\
			&\leq \left\|\ones_k-\dfrac{\matA\EE\matA^T}{\nu}\ones_k\right\|_2^2 + \left\|\dfrac{\matA(\matA^T-\EE\matA^T)}{\nu}\ones_k\right\|_2^2.\end{align*}

			Note that $\EE\matA = \frac{s}{k}\ones_{k\times r}$. Therefore, $\ones_k = \frac{k}{rs}\EE\matA\ones_r$. Using this fact and taking $\nu = \frac{rs^2}{k}$, we get

			\begin{align*}
			{\bf u}_1 &\leq \left\|\ones_k-\dfrac{\matA\EE\matA^T}{\nu}\ones_k\right\|_2^2 + \left\|\dfrac{\matA(\matA^T-\EE\matA^T)}{\nu}\ones_k\right\|^2\\
			&\leq \left\|\frac{k}{rs}\EE\matA\ones_r-\dfrac{s\matA}{\nu}\ones_r\right\|_2^2 + \frac{1}{\nu^2}\|\matA\|_2^2\|\matA-\EE\matA\|_2^2\|\ones_k\|_2^2\\
			&\leq \frac{k^2}{rs^2}\|\matA-\EE\matA\|_2^2 + \frac{k^3}{r^2s^4}\|\matA\|_2^2\|\matA-\EE\matA\|_2^2.
			\end{align*}

			Note that since $\EE\matA= \frac{s}{k}\ones_{k\times r}$, we have
			\begin{gather*}
			\|\EE\matA\|_2 = \frac{s}{k}\|\ones_{k\times r}\|_2 = \frac{s}{k}\sqrt{kr} = \sqrt{1-\delta}s.\end{gather*}

			Using our assumption that $\|\matA-\EE\matA\|_2 \leq \gamma \leq \sqrt{(1-\delta)}s$, we find
			\begin{align*}
			\|\matA\|_2^2 &\leq (\|\EE\matA\|_2+\gamma)^2\\
			&\leq (2\sqrt{(1-\delta)}s)^2\\
			&\leq 4(1-\delta)s^2.\end{align*}

			Finally, this implies
			\begin{align*}
			{\bf u}_1 &\leq\frac{k^2}{rs^2}\|\matA-\EE\matA\|_2^2 + \frac{k^3}{r^2s^4}\|\matA\|_2^2\|\matA-\EE\matA\|_2^2\\
			&\leq \frac{k^2\gamma^2}{rs^2} + \frac{4k^3(1-\delta)s^2\gamma^2}{r^2s^2}\\
			&=\frac{5\gamma^2k}{(1-\delta)s^2}.\end{align*}\end{proof}

			Therefore, to bound $\err_1(\matA)$ or ${\bf u}_1$, it suffices to bound $\|\matA-\EE\matA\|_2$. Moreover, the bounds we get by either method are within a constant factor of each other. By \cite{Furedi1981}, if $s \gg \log^4(k)$, then $\matA$ will concentrate well around $\EE\matA$. This bound was later improved by the following result from \cite{lei2015consistency}.

			\begin{lemma}\label{dense_conc}Suppose we have a random Erd\H{o}s-R\'enyi graph $G(n,p)$ with adjacency matrix ${\bf B}$ where $np \geq \log(n)$. For any $\alpha \geq 1$ there exists a universal constant $C_1 = C_1(\alpha)$ such that with probability at least $1-n^{-\alpha}$,
			$$\|{\bf B} - \EE {\bf B}\|_2 \leq C_1\sqrt{np}.$$\end{lemma}
			
			More generally, assume that ${\bf B}$ is a $n\times n$ adjacency matrix where ${\bf B}_{i,j}$ is Bernoulli with probability $p_{i,j}$. This is sometimes referred to as the {\it inhomogeneous} Erd\H{o}s-R\'enyi model $G(n,(p_{i,j}))$. Let $p = \max_{i,j} p_{i,j}$. As discussed in \cite{le2017concentration}, Lemma \ref{dense_conc} extends to this setting using this definition of $p$ (see section 1.1). While this result applies directly to $n\times n$ adjacency matrices, we can easily extend this to $\matA$. This will first require a basic lemma about the spectral norm of a structured block matrix.

			\begin{lemma}\label{block_spec}Let ${\bf D}$ be a $n_1\times n_2$ matrix. Suppose that ${\bf C}$ is a $n\times n$ block matrix of the form
			$${\bf C} = \begin{pmatrix} 0 & {\bf D}\\ {\bf D}^T & 0\end{pmatrix}.$$
			Then $\|{\bf C}\|_2 = \|{\bf D}\|_2$.\end{lemma}
			\begin{proof}
			Standard properties of singular values imply that $\|{\bf D^TD}\|_2 = \|{\bf DD^T}\|_2 = \|{\bf D}\|_2^2$. Moreover,
			$${\bf C}{\bf C}^T = \begin{pmatrix}{\bf D}{\bf D}^T & 0\\ 0 & {\bf D}^T{\bf D}\end{pmatrix}.$$
			Since the eigenvalues of a block diagonal matrix are given by the eigenvalues of  all the blocks,
			$$\|{\bf C}\|_2^2 = \|{\bf C}{\bf C}^T\|_2 = \max\{ \|{\bf D^TD}\|_2, \|{\bf DD^T}\|_2\} = \|{\bf D}\|_2^2.$$
			\end{proof}

			\begin{theorem}\label{conc_nonreg}Let $\matA$ be a $k\times r$ matrix where $k \geq r$ and $\matA_{i,j}$ is Bernoulli with probability $s/k$. Then for all $\alpha \geq 1$, there exists a universal constant $C_2 = C_2(\alpha)$ such that with probability at least $1-(k+r)^{-\alpha}$,
			$$\|\matA-\EE\matA\|_2 \leq C_2\sqrt{s}.$$\end{theorem}

			\begin{proof}$\matA$ encodes the structure of a bipartite graph with $k+r$ vertices. After relabeling, we can denote these vertices as $v_1,\ldots, v_k,v_{k+1},\ldots, v_{k+r}$ where the bipartite blocks are given by $\{v_1,\ldots, v_k\}, \{v_{k+1},\ldots, v_{k+r}\}$.  The adjacency matrix ${\bf B}$ is therefore of the form
			$${\bf B} = \begin{pmatrix} 0 & \matA \\ \matA^T & 0\end{pmatrix}.$$
			Note that ${\bf B}$ comes from an inhomogeneous Erd\H{o}s-R\'enyi graph $G(k+r,(p_{i,j}))$ where $p_{i,j}$ is zero if $i$ and $j$ are both in $\{1,\ldots, k\}$ or $\{k+1,\ldots, k+r\}$, and $s/k$ otherwise. Therefore, $p = \max_{i,j}p_{i,j} = s/k$. By Lemma \ref{dense_conc} (and the discussion following it), for all $\alpha > 0$ there exists some universal constant $C_1 = C_1(\alpha)$ such that with probability at least $1-(k+r)^{-\alpha}$,
			$$\|{\bf B}-\EE{\bf B}\|_2 \leq C_1\sqrt{\frac{(k+r)s}{k}} \leq \sqrt{2}C_1\sqrt{s}.$$
			Here, we used the fact that $r \leq k$. Note that $\EE{\bf B}$ satisfies
			$$\EE{\bf B} = \begin{pmatrix} 0 & \EE\matA \\ \EE\matA^T & 0\end{pmatrix}.$$
			Therefore,
			\begin{align*}
			\|{\bf B}-\EE{\bf B}\|_2 &= \left\|\begin{pmatrix} 0 & \matA-\EE\matA\\ (\matA-\EE\matA)^T & 0\end{pmatrix}\right\|_2\\
			&= \|\matA-\EE\matA\|_2.\end{align*}
			This last equality holds by Lemma \ref{block_spec}. Taking $C_2 = \sqrt{2}C_1$ we conclude the proof.\end{proof}

			Combining this with Lemma \ref{ebound}, we get the following theorem.

			\begin{theorem}\label{main_thm_2}Suppose that $s \geq \log(n)$. Then for any $\alpha \geq 1$, there is a universal constant $C_2 = C_2(\alpha)$ such that for $\rho = \frac{k}{rs}$, with probability at least $1-(k+r)^{-\alpha}$,
			$$\err_1(\matA) \leq \dfrac{C_2^2k}{(1-\delta)s}.$$\end{theorem}

			\begin{remark}Empirically, the same bound holds for other methods of generating $\matG$. If we choose the non-zero support of each column by selecting $s$ indices with or without replacement from $\{1,\ldots, k\}$, then we conjecture that the same theorem holds. Unfortunately, standard concentration inequalities are not enough to prove this result in such settings.\end{remark}

		\subsection{Regularized Bernoulli Gradient Codes}

			Bernoulli codes have two issues when $s < \log(k)$. First, each column only computes $s$ tasks in expectation, but may have columns with degree up to $s+\log(k)$. If $s \geq \log(k)$, this is not an issue as this gives us sparsity that is $O(s)$. When $s < \log(k)$, however, this may be an issue. Second, if $s \ll \log(k)$, $\matA$ may not concentrate around $\EE\matA$. For example, if $s = O(1)$ then by \cite{Krivelevich2003}, 
			$$\|\matA\|_2 = (1+o(1))\sqrt{\dfrac{\log k}{\log\log k}}.$$
			On the other hand, $\Exp\matA = p\ones_{k\times r}$ so $\|\Exp\matA\|_2 = p\sqrt{kr} = \sqrt{1-\delta}s$. For $s \ll \log(k)$, this implies that $\matA$ does not concentrate as well. Therefore we cannot use Lemma \ref{ebound} to bound $\err(\matA)$.

			Both of these issues have the same cause: vertices whose degree is too large. Fortunately, this issue of enforcing concentration of sparse graphs has been studied and partially resolved in \cite{le2017concentration}. They show that by appropriate regularization of graphs, we can improve their concentration in the sparse setting.

			\begin{theorem}[\hspace{1sp}\cite{le2017concentration}]Let ${\bf B}$ be a random graph from the inhomogeneous Erd\H{o}s-R\'enyi model $G(n,(p_{i,j}))$ and let $p = \max_{i,j}p_{i,j}$. For any $\alpha \geq 1$, the following holds with probability at least $1-n^{-\alpha}$. Take all vertices of ${\bf B}$ with degree larger than $2np$ and reduce the weights of the edges incident to those vertices in any way such that they have degree at most $np$. Let $\matB'$ denote the resulting graph. Then
			$$\|\matB'-\EE\matB\|_2 \leq C_3r^{3/2}\sqrt{np}.$$\end{theorem}

			Here $C_3$ is a universal constant. Note that this regularization can be performed analogously on $\matA$. To form $\matA'$, we simply look at all columns with degree more than $2s$ and change entries in those columns from $1$ to $0$ until these columns have degree $s$. This satisfies the criterion in the above theorem. We can then use an almost identical version of the proof of Theorem \ref{conc_nonreg} to prove the following theorem.

			\begin{theorem}\label{conc_thm}There is a universal constant $C_4 = C_4(\alpha)$ such that for any $\alpha \geq 1$, $s \geq 1$, with probability at least $1-(k+r)^{-\alpha}$,
			$$\|\matA'-\Exp \matA\|_2 \leq C_3\alpha^{3/2}\sqrt{s}.$$
			\end{theorem}

			We can combine this with Lemma \ref{ebound} to derive the following theorem concerning $\err(\matA')$. As before, this bound applies for both the one-step decoding and the optimal decoding.

			\begin{theorem}\label{main_thm_4}
			For any $\alpha \geq 1, s \geq 1$, and letting $\rho = \frac{k}{rs}$, with probability at least $1-(k+r)^{-\alpha}$,
			$$\err_1(\matA') \leq \frac{C_3^2\alpha^{3}k}{(1-\delta)s}.$$
			\end{theorem}

			By regularizing in the above manner, we ensure that each compute node computes at most $2s$ tasks and that our error bound works for all $s$. Note that in practice, we cannot form $\matA'$ from $\matA$ as we don't know $\matA$ a priori. Therefore we cannot tell the compute nodes to compute the functions corresponding to $\matA'$. Instead, we can regularize $\matG$ in the same manner to obtain $\matG'$ in the following way such that we can apply Theorem \ref{main_thm_4} above. We refer to this code as the {\it regularized Bernoulli Gradient Code}, (rBGC).

			The construction is simple. We initialize ${\bf G}$ with each entry Bernoulli$(s/k)$. For each column $j$ with more than $2s$ non-zero entries, we randomly set entries to 0 until it has $s$ non-zero entries. A detailed algorithm is provided below.

			Note that the error incurred by an rBGC corresopnds to a multiplicative error $\epsilon$ where
			$$\epsilon = O\left(\dfrac{1}{(1-\delta)s}\right).$$
			This therefore implies Theorem \ref{main_thm_bgc}.			

			\medskip

			{\begin{minipage}{.9\linewidth}
			\flushleft
				\begin{algorithm}[H]
				    \SetKwInOut{Input}{Input}
				    \SetKwInOut{Output}{Output}

				    \Input{$n,k,s$.}
				    \Output{A $k\times n$ regularized function assignment matrix ${\bf G}'$ with max degree $\leq 2s$}
				    ${\bf G}' = {\bf 0}_{k\times n}.$;\\
				    \For{$j = 1$ \KwTo $n$}{
				    	$d = 0$;\\

				    	\For{$i = 1$ \KwTo $k$}{
				    		${\bf G}_{i,j}' = Bernoulli(s/k)$;\\
				    		$d = d + {\bf G}_{i,j}'$;
				    	}
				    	\If{$d > 2s$}{
				    		\While{$d > s$}{
				    			remove a random edge from column $j$;\\
				    			$d = d - 1$;\\
				    		}
				    	}
				    }
				 	return ${\bf G}'$;
				    \caption{Regularized Bernoulli Gradient Code.}
				\end{algorithm}
			\end{minipage}
			\par}

	\section{Simulations}

		\subsection{Decoding Error of Various Coding Schemes}

			In this section we compare the empirical decoding error of Fractional Repetition Codes (FRCs) and Bernoulli Gradient Codes (BGC). Recall that we gave two decoding methods, one that corresponding to the optimal decoding error
			$$\err(\matA) = \min_{\vecx}\|\matA\vecx-\ones_k\|_2^2$$
			and one corresponding to the one-step decoding error
			$$\err_1(\matA) = \|\rho\matA\ones_r-\ones_k\|_2^2.$$
			We compare FRCs and BGCs to a coding scheme proposed in \cite{raviv2017gradient}. There, Raviv et al. consider the scheme where ${\bf G}$ is the adjacency matrix of an $s$-regular expander graph with $k$ vertices. They show that for all $k\times r$ submatrices $\matA$, 
			$$\err_1(\matA) \leq \dfrac{\lambda({\bf G})^2}{s^2}\dfrac{\delta k}{(1-\delta)}.$$
			Here, $\lambda({\bf G}) = \max\{|\lambda_2|,|\lambda_k|\}$, where the eigenvalues of ${\bf G}$ are given by
			$$\lambda_1 \geq \lambda_2 \geq \ldots \geq \lambda_k.$$
			We would like to construct ${\bf G}$ to have $\lambda({\bf G})$ as small as possible. This is achieved by {\it Ramanujan graphs}. In practice, constructing expander graphs with small values of $\lambda$ is difficult. By taking a random $s$-regular graph, however, we can obtain can expander graph with high probability \cite{lubotzky2012expander}. As $k \to \infty$, $\lambda$ tends to the optimal value. In order to generate empirical data, we consider the setting where ${\bf G}$ is the adjacency matrix of a random $s$-regular graph.

			Below, we plot the one-step and optimal decoding error $\err(\matA)$ and $\err_1(\matA)$ for these three schemes when $k = 100$ and the fraction of stragglers $\delta$ varies. In order to normalize the error, we plot $\err(\matA)/k$ and $\err_1(\matA)/k$. We take $\rho = \frac{k}{rs}$ in the one-step decoding.

			\begin{figure}[H]
			\captionsetup{width=0.9\textwidth}
			\centering
			\begin{subfigure}{.45\textwidth}
			  \centering
			  \includegraphics[width=\linewidth]{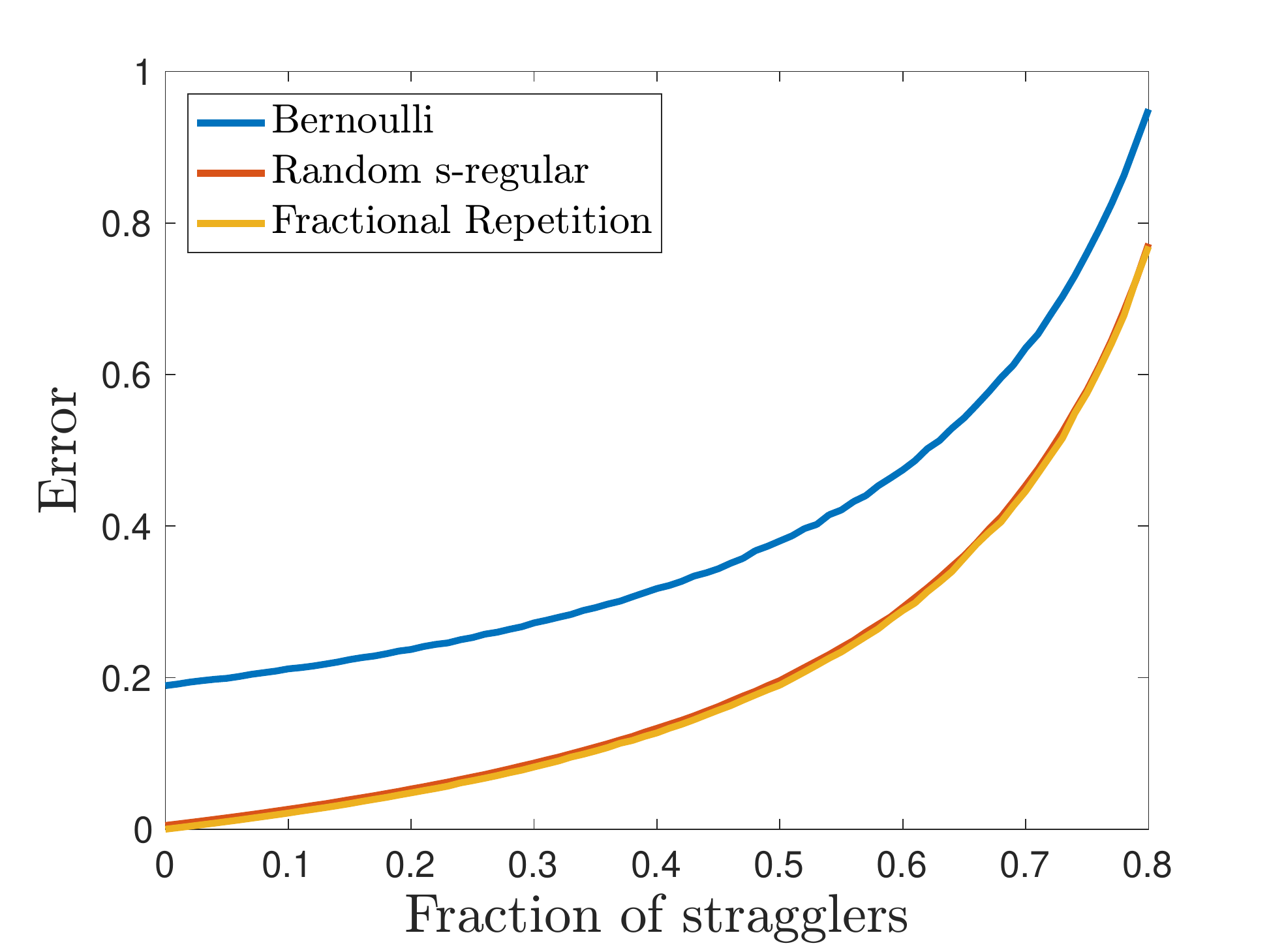}
			  \caption{$s = 5$}
			  \label{fig:err1_s5}
			\end{subfigure}%
			\begin{subfigure}{.45\textwidth}
			  \centering
			  \includegraphics[width=\linewidth]{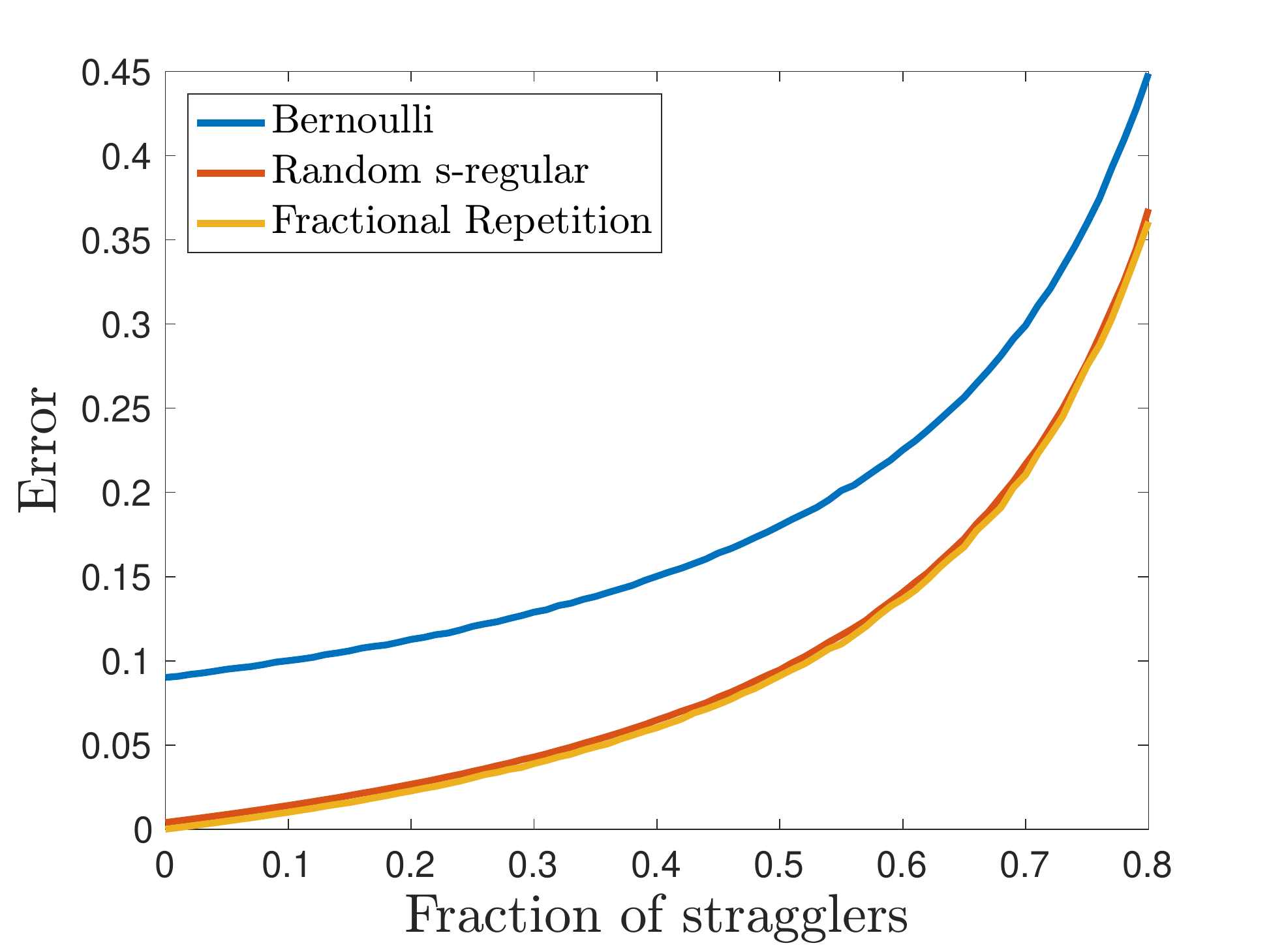}
			  \caption{$s = 10$}
			  \label{fig:err1_s10}
			\end{subfigure}
			\caption{A plot of the average one-step error $\err_1(\matA)/k$ over 5000 trials. We take $k = 100$, $r = (1-\delta)k$ for varying $\delta$. The figure on the left has $s = 5$ while the figure on the right has $s = 10$.}
			\label{fig:compare}
			\end{figure}
			We see that under one-step decoding, FRCs and $s$-regular expanders perform extremely comparably. In this setting, BGCs seem to sacrifice some accuracy for simplicity. However, FRCs are also computationally simple and perform as well as taking $s$-regular expanders in the average case under one-step decoding. For optimal decoding, FRCs perform significantly better than $s$-regular expanders or BGCs, as the following plots show.

			\begin{figure}[H]
			\captionsetup{width=0.9\textwidth}
			\centering
			\begin{subfigure}{.45\textwidth}
			  \centering
			  \includegraphics[width=\linewidth]{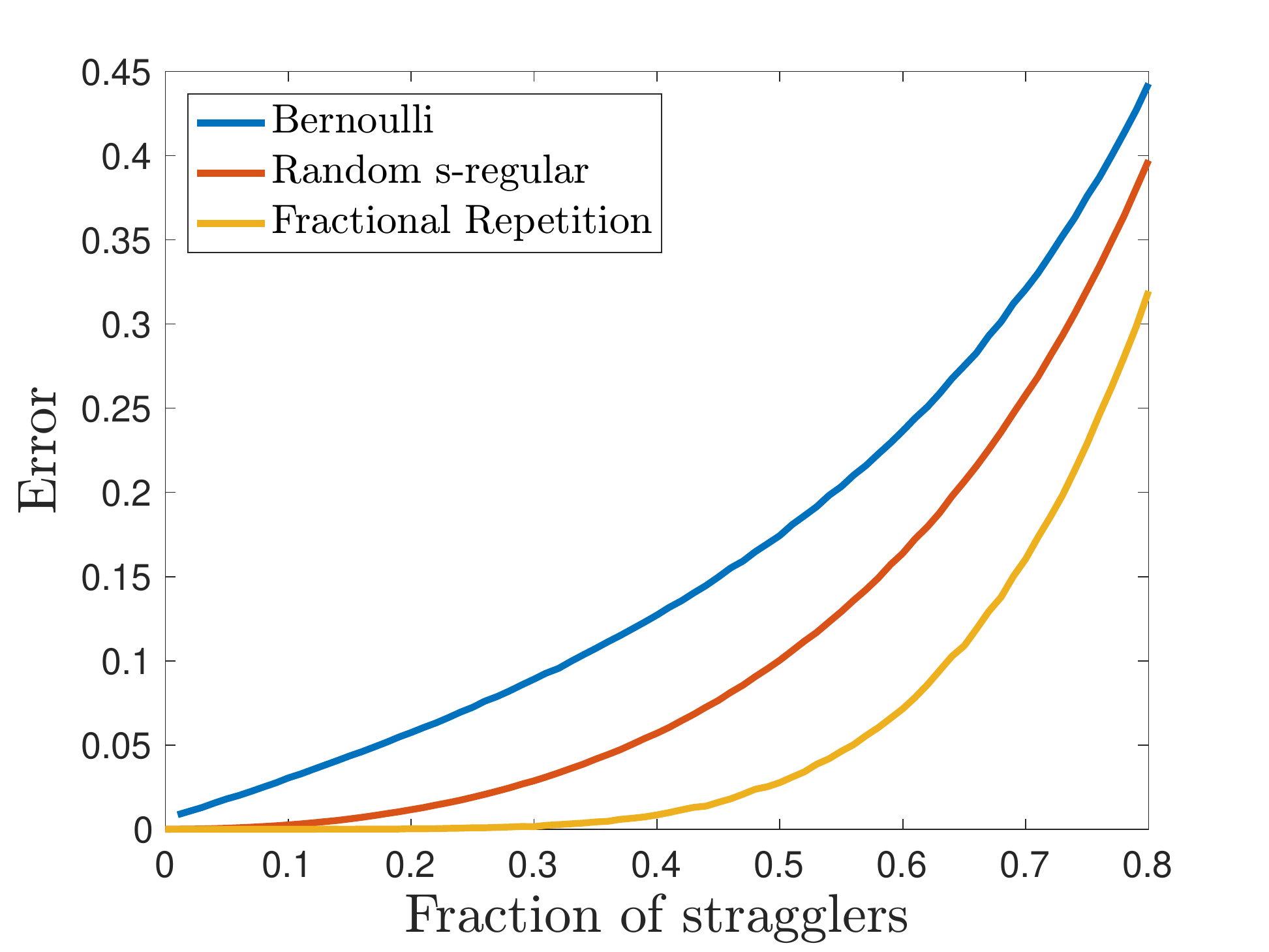}
			  \caption{$s = 5$}
			  \label{fig:err_s5}
			\end{subfigure}%
			\begin{subfigure}{.45\textwidth}
			  \centering
			  \includegraphics[width=\linewidth]{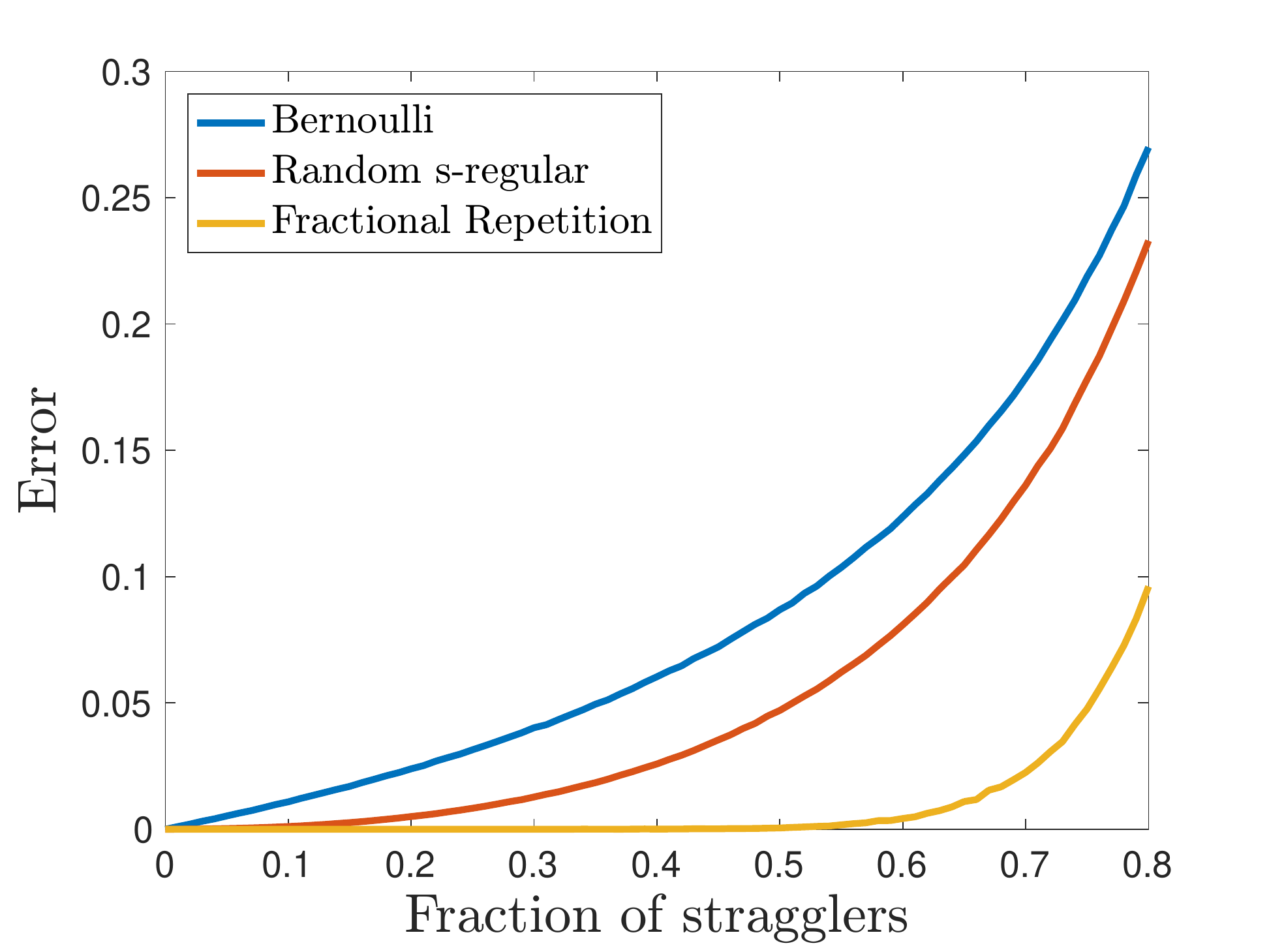}
			  \caption{$s = 10$}
			  \label{fig:err_s10}
			\end{subfigure}
			\caption{A plot of the average optimal decoding error $\err(\matA)/k$ over 5000 trials. We take $k = 100$, $r = (1-\delta)k$ for varying $\delta$. The figure on the left has $s = 5$ while the figure on the right has $s = 10$.}
			\label{fig:compare}
			\end{figure}		
			These plots show that if we instead consider optimal decoding, then FRCs greatly outperform the other two methods. In particular, FRCs can achieve zero optimal decoding error even with a non-trivial fraction of stragglers. If $s = 10$, then we can achieve close to zero error even with half of the compute nodes being stragglers.

			Finally, we compare the one-step and optimal decoding error for the BGCs, FRCs, and $s$-regular graphs. The results are plotted below.

			\begin{figure}[H]

				\begin{subfigure}{\textwidth}
					\captionsetup{width=0.9\textwidth}
					\centering
					\begin{subfigure}{.3\textwidth}
					  \centering
					  \includegraphics[width=\linewidth]{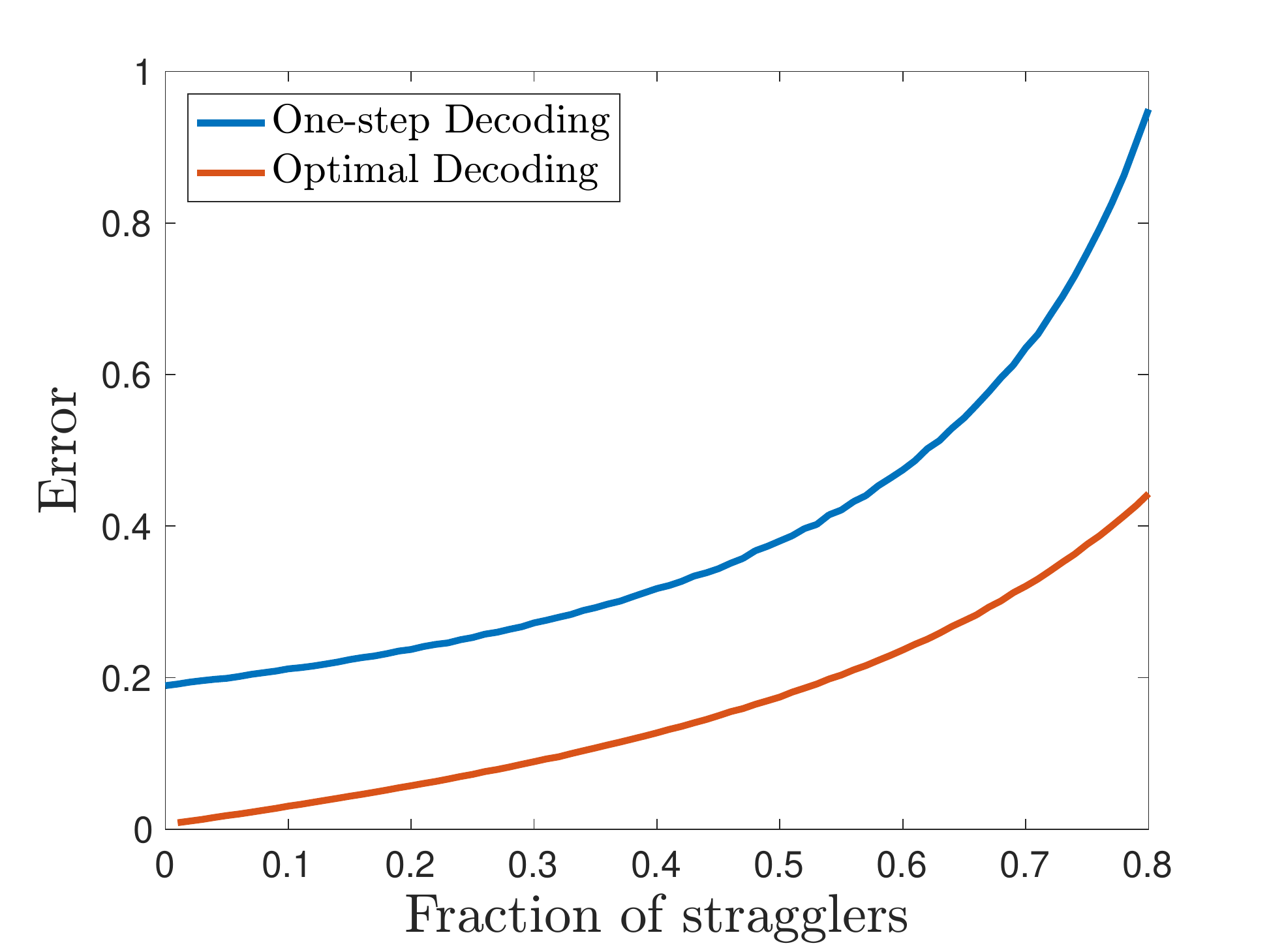}
					  \label{fig:bgc_s5}
					\end{subfigure}%
					\begin{subfigure}{.3\textwidth}
					  \centering
					  \includegraphics[width=\linewidth]{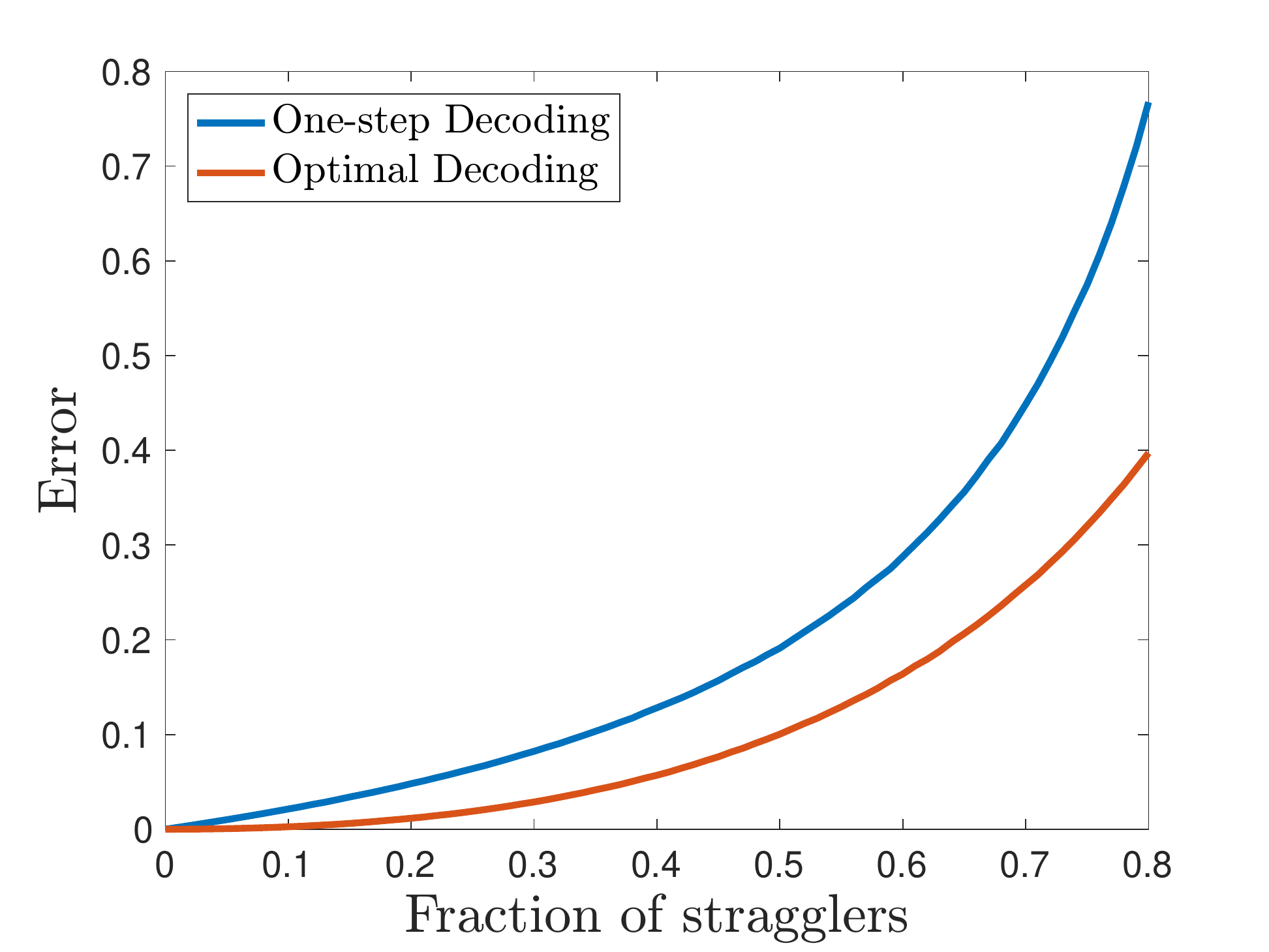}
					  \label{fig:sreg_s5}
					\end{subfigure}%
					\begin{subfigure}{.3\textwidth}
					  \centering
					  \includegraphics[width=\linewidth]{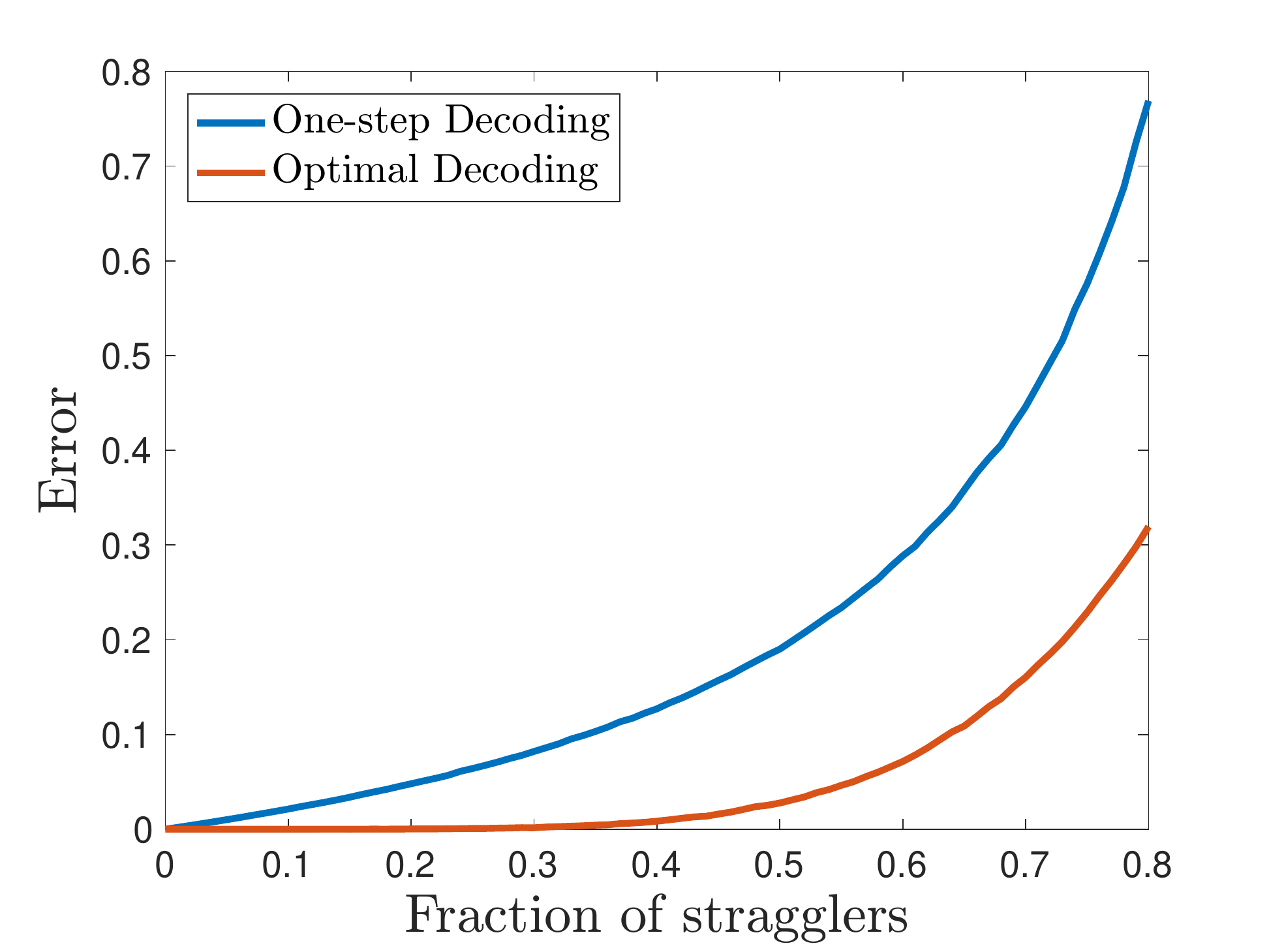}
					  \label{fig:sreg_s5}
					\end{subfigure}%
				\end{subfigure}\\	
				\begin{subfigure}{\textwidth}
					\captionsetup{width=0.9\textwidth}
					\centering
					\begin{subfigure}{.3\textwidth}
					  \centering
					  \includegraphics[width=\linewidth]{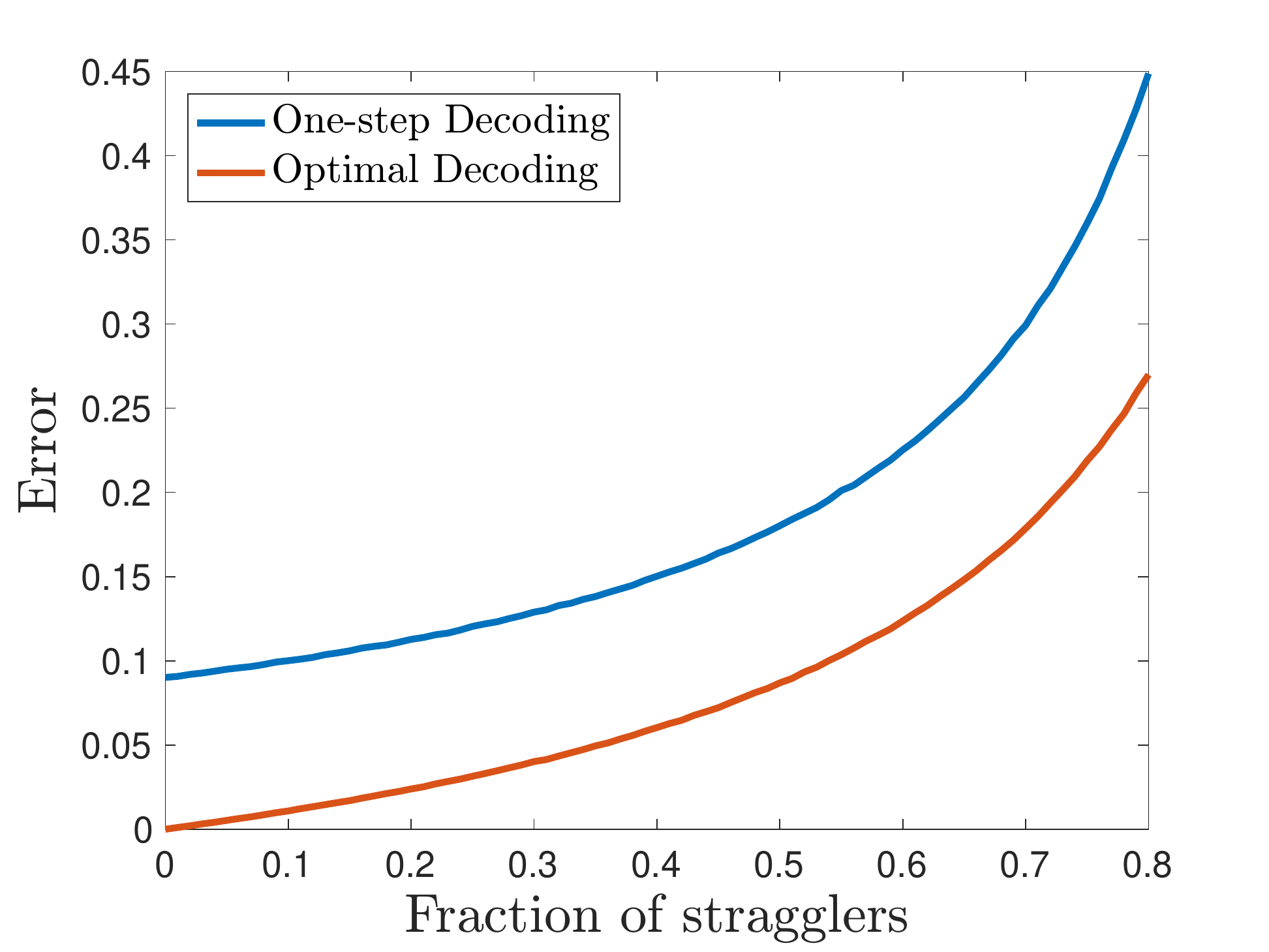}
					  \caption{Bernoulli Gradient Code}
					  \label{fig:bgc_s10}
					\end{subfigure}%
					\begin{subfigure}{.3\textwidth}
					  \centering
					  \includegraphics[width=\linewidth]{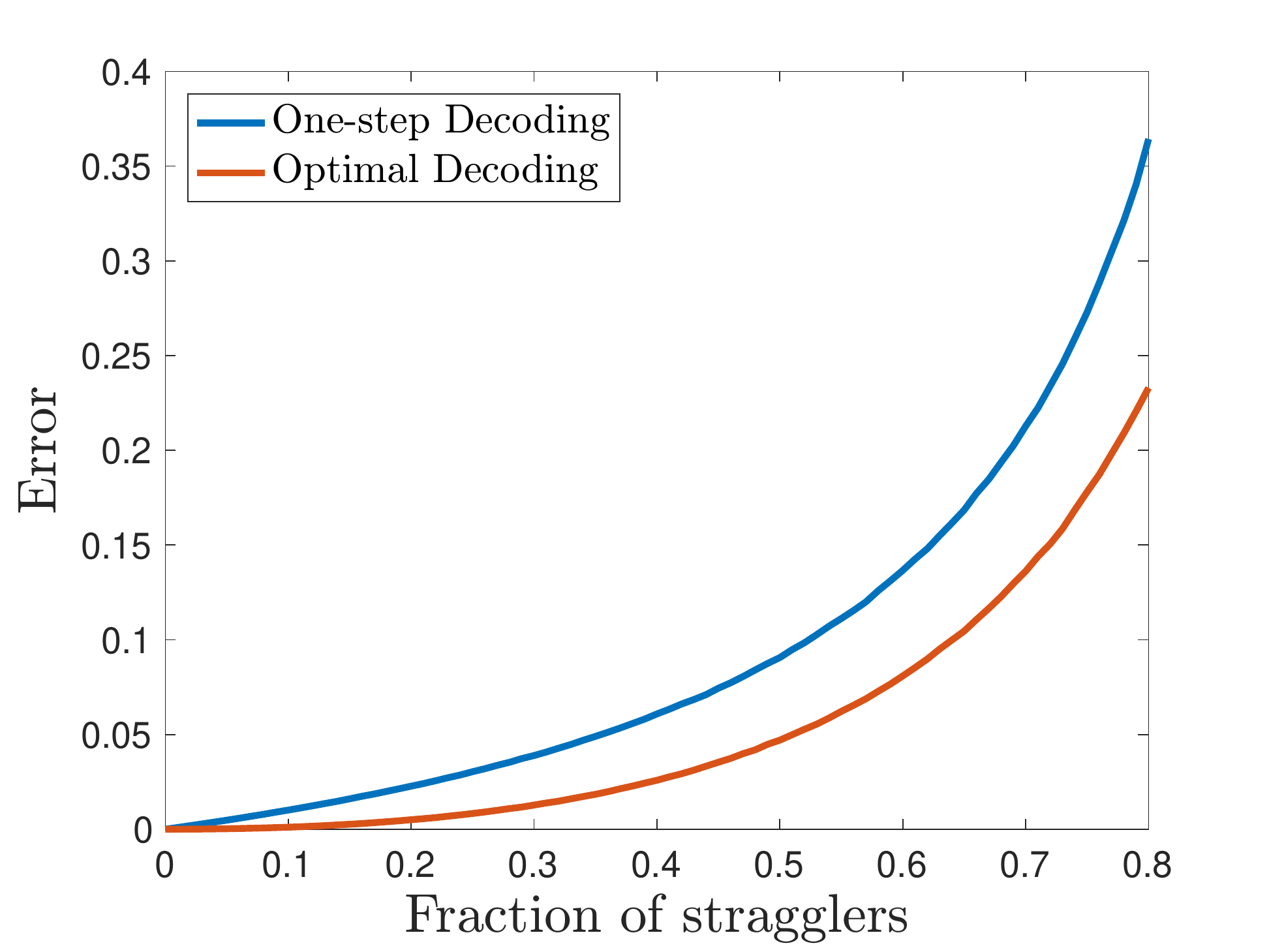}
					  \caption{$s$-Regular Graphs}
					  \label{fig:sreg_s10}
					\end{subfigure}%
					\begin{subfigure}{.3\textwidth}
					  \centering
					  \includegraphics[width=\linewidth]{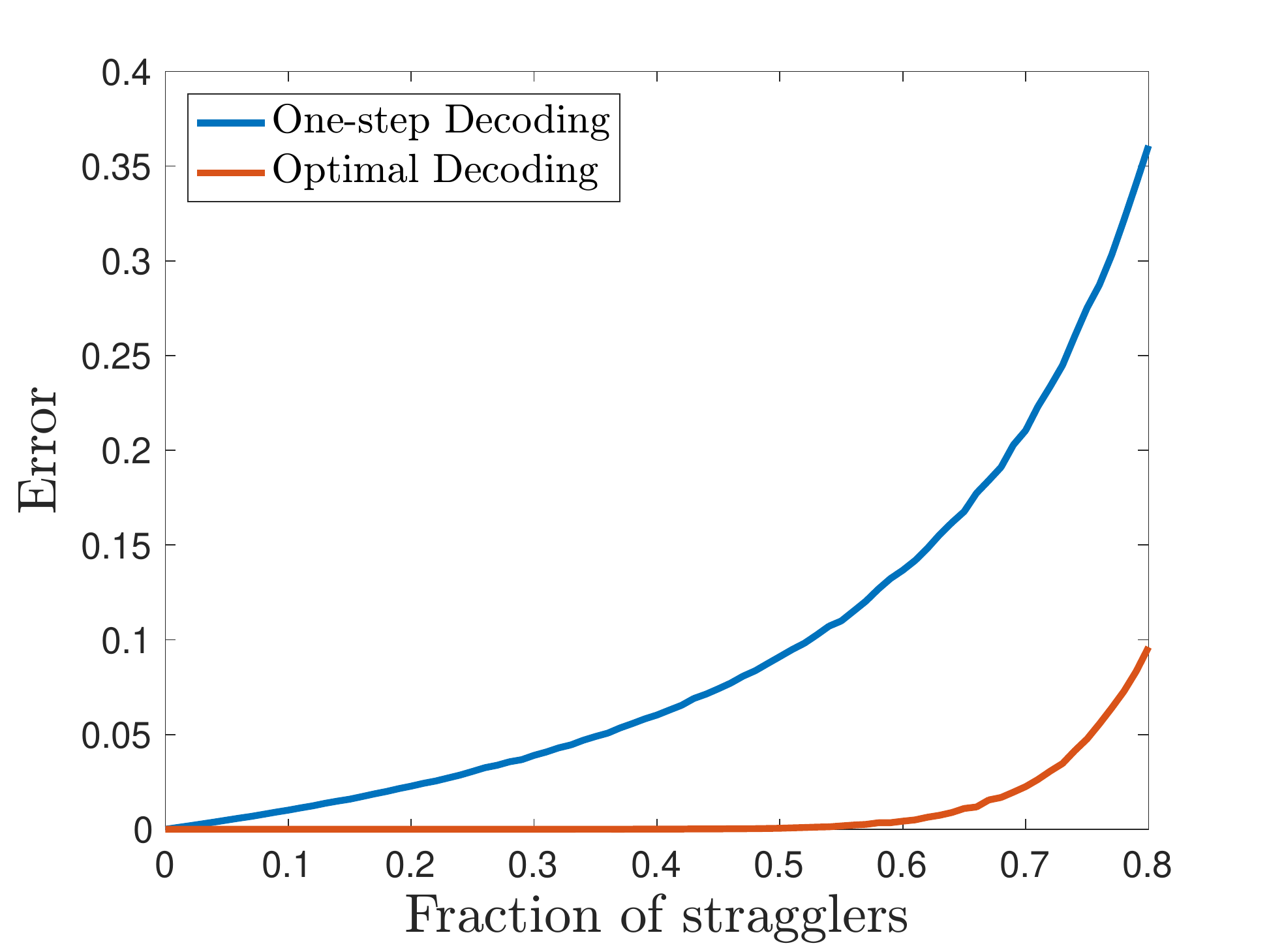}
					  \caption{Fractional Repetition Code}
					  \label{fig:sreg_s10}
					\end{subfigure}%
					\label{fig:compare_s10}
				\end{subfigure}
			\caption{A plot of the average optimal decoding error $\err(\matA)/k$ and one-step decoding error $\err_1(\matA)/k$ over 5000 trials. We take $k = 100, r = (1-\delta)k$ for varying $\delta$. We compare these two decoding errors for BGCs (left), $s$-regular graphs (middle), and FRCs (right). The top plots have $s = 5$, while the bottom have $s = 10$.}
			\label{fig:compare_all}				
			\end{figure}			

		\subsection{Algorithmic Decoding Error of Bernoulli Gradient Codes}

			In this section we give empirical one-step error rates of BGCs for varying sizes of $r$ and sparsity $s$ of the matrix $\matA$ . Recall that we defined the algorithmic decoding error of $\matA$ by a sequence of vectors ${\bf u}_t$. Here, $\|{\bf u}_1\|_2^2$ corresponds to the one-step decoding error, while $\|{\bf u}_t\|_2^2$ converges to the optimal decoding error.

			Let ${\bf G}$ be constructed via a BGC. We fix $k = 100$ and take varying values of $\delta$ and $s$, letting $r = (1-\delta)k$. We then calculate the average value of $\|{\bf u}_t\|_2^2/k$ based on a Monte Carlo simulation for increasing values of $t$. We set $\rho = \|\matA\|_2^2$. The results are below.

			\begin{figure}[H]
			\captionsetup{width=0.9\textwidth}
			\centering
			\begin{subfigure}{.45\textwidth}
			  \centering
			  \includegraphics[width=\linewidth]{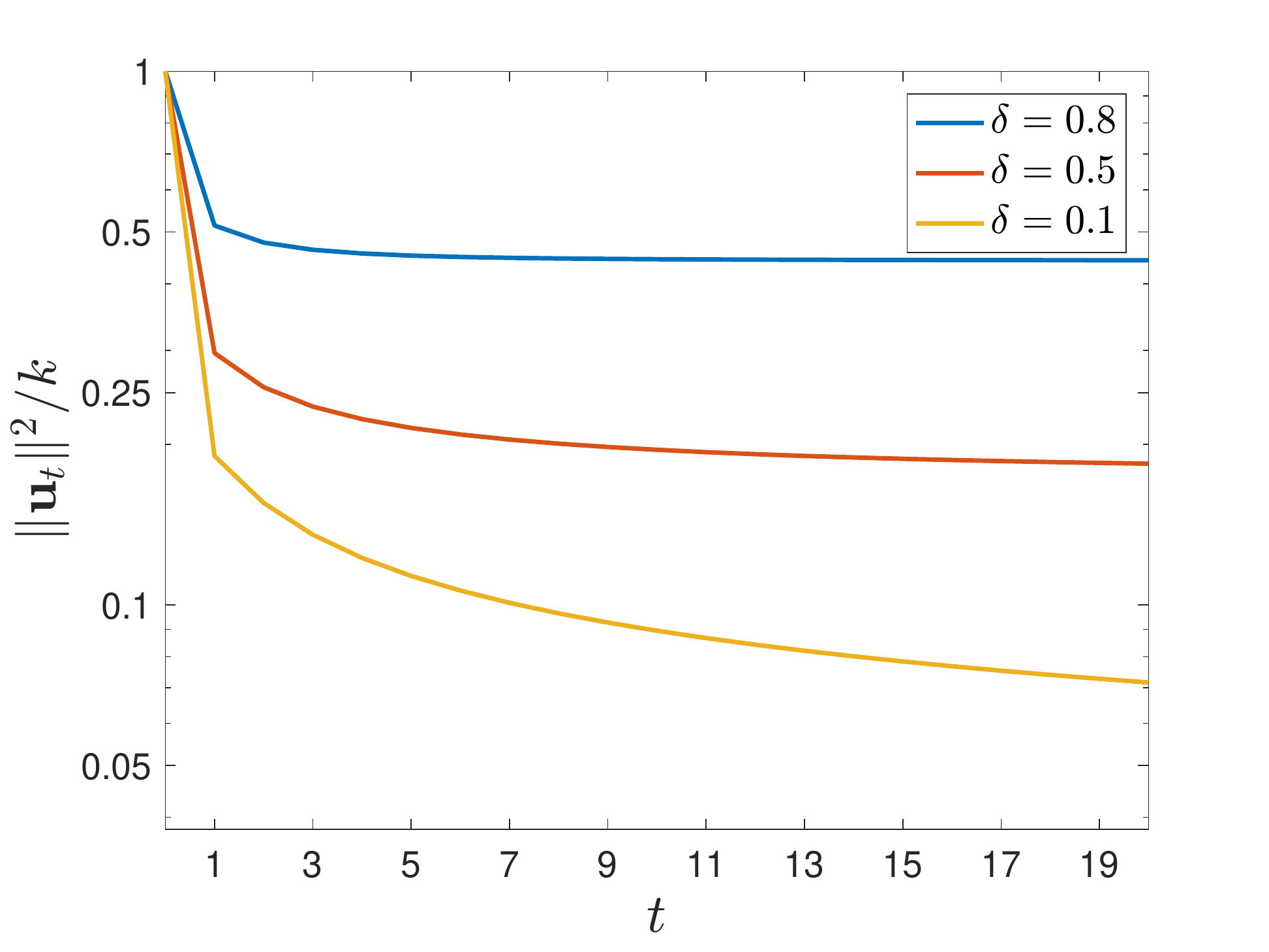}
			  \caption{$s = 5$}
			  \label{fig:sub1}
			\end{subfigure}%
			\begin{subfigure}{.45\textwidth}
			  \centering
			  \includegraphics[width=\linewidth]{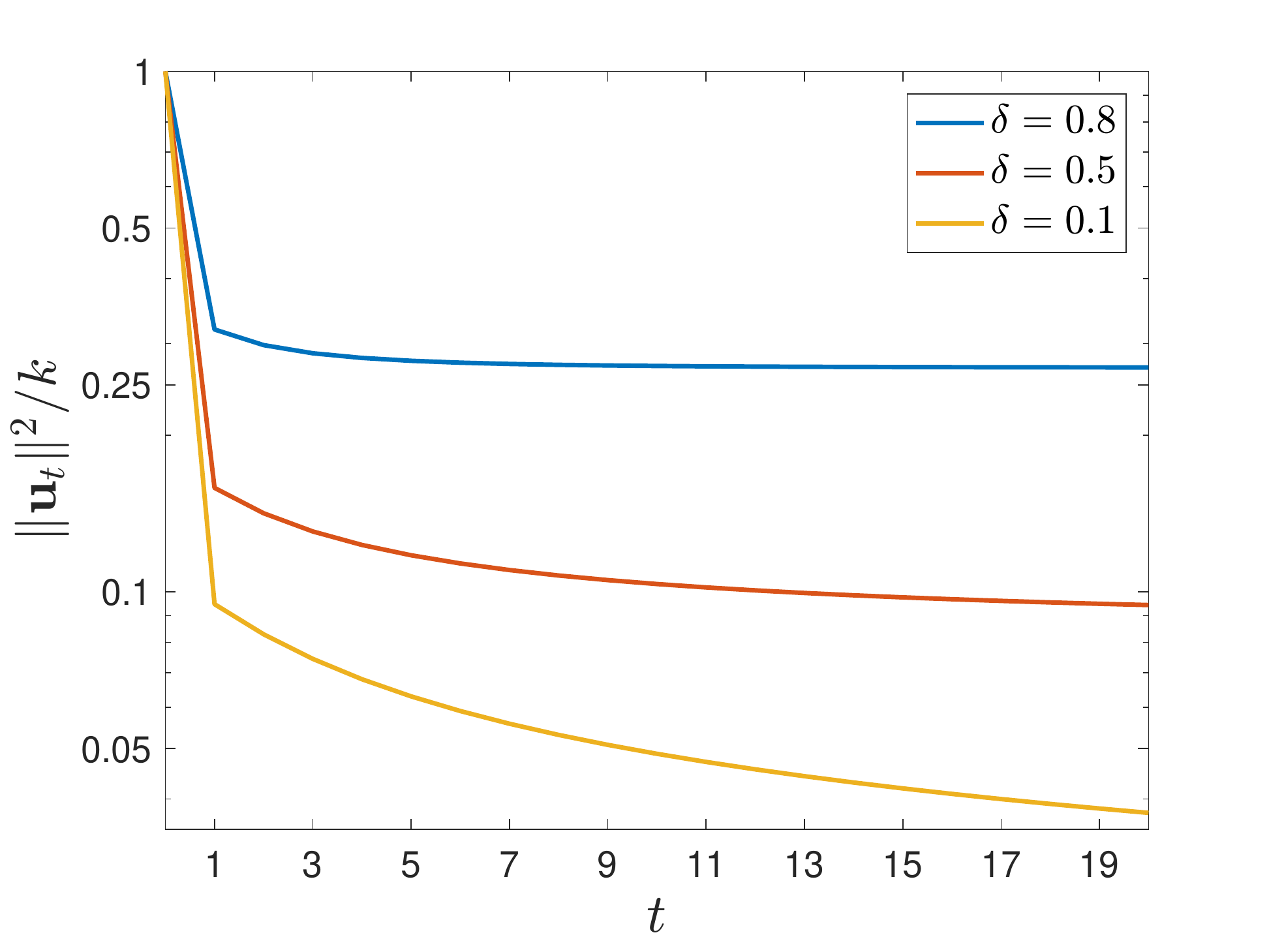}
			  \caption{$s = 10$}
			  \label{fig:sub2}
			\end{subfigure}
			\caption{The average value of $\|{\bf u}_t\|_2^2/k$ of a BGC for $\delta \in \{0.1, 0.2, 0.3, 0.5, 0.8\}$ and varying $t$ for 5000 trials. The figure on the left plots the algorithmic error for sparsity $s = 5$, while the figure on the right plots the algorithmic error for sparsity $s = 10$.}
			\label{fig:test}
			\end{figure}		

	\section{Conclusion and Open Problems}

	In this work, we formally described the approximate gradient coding problem and gave two different decoding methods for such codes. We analyzed two efficiently computable gradient codes, FRCs and BGCs, and gave explicit bounds for their decoding error. While FRCs exhibit extremely low average-case error, they are susceptible to adversaries, even polynomial-time adversaries. This is in contrast to our result that adversarial straggler selection is NP-hard. BGCs are constructed via sparse random graphs and are potentially less susceptible to polynomial-time adversaries.

	While the problem of approximate gradient coding can be stated in relatively simple terms, our work shows that the problem is not simple. Bounding the error of gradient codes, even relatively straightforward codes such as BGCs, can be an arduous task. This is especially true in the context of optimal decoding error. Still, approximate gradient codes have exciting connections to coding theory, concentration of random graphs, expander graphs, and many other interesting combinatorial and algebraic objects. This work is intended to be a starting point towards understanding the approximate gradient coding problem. There are many remaining open problems, some directly continuing this work, others more tangentially related.

		\bigskip

		{\bf Tighter bounds on the optimal decoding error:} For both BGCs and $s$-regular expander graphs, our empirical results show that there is a significant gap between the one-step and the optimal decoding error. So far, the only known results on the error of these codes concern their one-step decoding error. Better bounds on their optimal decoding error and, more generally, better methods for bounding the optimal decoding error would have great implications for the analysis and design of gradient codes. Unfortunately, such bounds are not straightforward. They require greater care and may require analysis of second- and higher-order moments or a better understanding of the pseudo-inverse of random matrices.

		\bigskip

		{\bf Connections to the Restricted Isometry Property:} Bounding the decoding error under the presence of $r$ non-stragglers amounts to bounding $\|{\bf G}\vecx - \ones_k\|_2^2$ when $\vecx$ is $r$-sparse and ${\bf G}$ is sparse. If we were instead interested in $\|{\bf G}\vecx-\ones_k\|_1$, then this would bear a strong resemblance to the Restricted Isometry Property under the $\ell_1$ norm (RIP-1). This connection is further bolstered by the fact that expander graphs, well known to have the RIP-1 property, appear to perform well as gradient codes. Further exploration of the connection between gradient codes and the RIP-1 property could lead to better techniques for bounding the decoding error and better gradient codes.

		\bigskip

		{\bf Algorithmic error and weighted walk counting:} One potential method towards better bounds on the optimal decoding error comes from Lemma \ref{ut_lemma} above. As we show, one can bound the optimal decoding error by a weighted alternating sum involving the number of walks on a bipartite graph. Tight bounds on the number of walks in a random bipartite graph or deterministic constructions for bipartite graphs with more explicit fomrulas for the number of walks could lead to much better bounds on the error.

		\bigskip

		{\bf Approximate gradient coding capacity:} Our work gives deterministic and randomized codes that achieve varying worst-case and average-case errors. In general, we know very little about how these error rates behave. One can ask: what is the information theoretic limit of an approximate gradient code, \eg what is the optimal tradeoff between the recovery error, the probability of recovery, and the sparsity of the code?

\section*{Acknowledgements}

The first author was supported in part by the National Science Foundation grant DMS-1502553.

\nocite{*}
\bibliographystyle{plain} 
\bibliography{code}

\end{document}